\documentclass[10pt]{article}
\usepackage[]{geometry}
\usepackage{subcaption}
\usepackage{graphicx} 
\usepackage{algorithm}
\usepackage{algpseudocode}
\usepackage{authblk}

\usepackage{bbm}
\usepackage{bm}
\usepackage{xcolor}
\usepackage{natbib}
\usepackage[colorinlistoftodos,prependcaption,textsize=tiny]{todonotes}
\usepackage[normalem]{ulem}
 \bibpunct[, ]{(}{)}{,}{a}{}{,}%
\usepackage{amsmath, amssymb, amsthm}
\usepackage{booktabs}
\usepackage{multirow}
\usepackage{setspace}
\usepackage{tikz}
\usepackage[colorlinks,citecolor=blue]{hyperref}
\newcommand{\R}{\ensuremath{\mathbb{R}}}

\newcommand{\coef}{\boldsymbol{w}}
\newcommand{\ignore}[1]{}

\newcommand{\fairreg}{(FR1)}
\newcommand{\fairabs}{(FR2)}
\newcommand{\basicmodel}{(NAT)}
\newtheorem{proposition}{Proposition}
\newtheorem{corollary}{Corollary}
\theoremstyle{definition}
\newtheorem{remark}{Remark}

\begin{document}
\title{Fair and Accurate Regression: \\ Strong Formulations and Algorithms}
\author{Anna Deza\thanks{Industrial Engineering and Operations Research, University of California, \texttt{annadeza@berkeley.edu}.}, Andr\'es G\'omez\thanks{Industrial and System Engineering, University of Southern California, \texttt{gomezand@usc.edu}.}, and Alper Atamt\"urk \thanks{Industrial Engineering and Operations Research, University of California, \texttt{atamturk@berkeley.edu}.}}
\date{December 2024}
\maketitle
\footnotetext[1]{This research was supported, in part, by NSF AI Institute for Advances in Optimization Award 2112533 and AFOSR Award FA9550-24-1-0086.}
\begin{abstract}
This paper introduces mixed-integer optimization methods to solve regression problems that incorporate fairness metrics. We propose an exact formulation for training fair regression models. To tackle this computationally hard problem, we study the polynomially-solvable single-factor and single-observation subproblems as building blocks and derive their closed convex hull descriptions. 
Strong formulations obtained for the general fair regression problem in this manner are utilized to solve the problem with a branch-and-bound algorithm exactly or as a relaxation to produce fair and accurate models rapidly. Moreover, to handle large-scale instances, we develop a coordinate descent algorithm motivated by the convex-hull representation of the single-factor fair regression problem to improve a given solution efficiently. Numerical experiments conducted on fair least squares and fair logistic regression problems show competitive statistical performance with state-of-the-art methods while significantly reducing training times. 
\end{abstract}
\section{Introduction}
Machine learning models are used to make decisions that can have significant implications on individuals and communities, such as bank loan approvals, college admissions, allocation of healthcare resources, and criminal sentencing \citep{le2022survey}. Building such models without considering sources of bias can lead to undesirable outcomes, including discriminative behavior and amplifying the inaccuracies inherent to biased data \citep{barocas2016big, chouldechova2017fair}. 

The detection of algorithmic biases in machine learning models against specific groups in real-world scenarios, such as gender, race, age, and socioeconomic has spurred a recent effort to develop fair prediction models \citep{hort2024bias}.
Existing approaches to promote fair predictions predominantly address binary classification, with a lesser focus on regression \citep{caton2024fairness}.
Fair machine learning is divided into three broad categories, defined by the stage in which fairness considerations occur: pre-processing, in-processing and post-processing methods.
Among these, only in-processing methods allow for explicit control of fairness requirements during model training. In-processing techniques, which are the focus of this work, typically outperform other methods and have the added advantage of not requiring knowledge of an individual's group identity at deployment, unlike post-processing methods, which is often a legal requirement in sensitive settings \citep{pessach2022review}.

Incorporating fairness considerations directly into a model's training often introduces non-convexities and non-smoothness \citep{cotter2019optimization}, leading to potentially intractable optimization problems. One common approach to mitigate these challenges is to design proxy fairness measures that are easier to optimize \citep{kamishima2012fairness, calders2013controlling, zafar2017fairnessB, berk2017convex, komiyama18a, do2022fair}. Another line of work consists of solving a series of convex problems, as proposed by \citet{olfat2018spectral} and \citet{agarwal2018reductions} for classification, and  \citet{agarwal2019fair} for regression. Although such approaches reduce the computational burden of model training, the approximations used for fairness compromise the trade-off between accuracy and fairness \citep{olfat2018spectral}. In this work, we address these limitations by formulating the fair regression training problem with exact fairness metrics to achieve a more rigorous accuracy vs. fairness trade-off. Moreover, for improved computations, we propose a novel strengthening of the model through an extended formulation and develop a coordinate descent algorithm to produce high-quality solutions quickly.

% Frameworks for exact fairness have been proposed for classification, \citep{ye2020unbiased}, and for regression \citep{ye2024distributionally}. In this work, we formulate the fair regression training problem with the incorporation of exact fairness metrics.

\subsection{Problem Statement}\label{sec:problem.setting}

We begin by formally stating the fair regression problems considered in the paper. Let the population be partitioned into two classes: a protected class and a non-protected class. We assume a training set made up of $m$ observations, with each data point having the form $(\bm{x}_i, y_i, a_i)$, where $\bm{x}_i \in \R^n$ is the feature vector, $y_i \in \mathcal{Y} \subseteq \R$ is the response, and $a_i \in \{0,1\}$ indicating if observation $i$ belongs to a protected individual or not. Let $\hat{y} \in \hat{\mathcal{Y}}$ denote the prediction of a model, $m_1$ denote the number of protected individuals in the training set, and define $m_0 = m-m_1$.

Fairness is often measured using \emph{distance to demographic parity} \citep{pessach2022review}:
\begin{equation} 
\mathrm{DP} = \max_{b \in \hat{\mathcal{Y}}}| \mathbb{P}\left(\hat{y} > b | a = 1\right) - \mathbb{P}\left(\hat{y} > b\right) |,\label{def.dp}
\end{equation}
where probabilities are taken over the distribution of the data.
Demographic parity measures the difference in prediction rates across the protected class and the entire distribution. Indeed, if $\mathrm{DP}= 0$, then $\hat{y} \perp a$, that is, predictions are statistically independent of protected the class.  Thus, the metric $\mathrm{DP}$ measures the deviation from demographic parity. A small DP translates to a lesser difference between the positive prediction rates among groups, minimizing disparate impacts.

This paper considers linear regression models, which are parameterized by regression coefficients $\coef \in \R^n$.  Given a convex loss function $\mathcal{L}:\R \times \mathcal{Y} \rightarrow \R$, a fair and accurate model corresponds to the optimal solution of the optimization problem
\begin{align}\label{fair.training.constr0}
        \min_{\coef \in \R^n} & \mathbbm{E}\left[\mathcal{L}\left(\coef^\top\bm{x}, y\right)\right]\text{ 
 s.t. }{\mathrm{DP}}(\coef) \leq \epsilon,
\end{align}
where expectation and probabilities are taken over the distribution of the data, and parameter $\epsilon \in [0,1]$ controls the maximal violation to demographic parity. 
Since the distribution is rarely known explicitly, \eqref{fair.training.constr0} is often approximated using the empirical distribution induced by the dataset $\{(\bm{x}_i, y_i, a_i)\}_{i=1}^m$, namely
\begin{subequations}
\label{fair.training.constr}
\begin{align}
        \min_{\coef \in \R^n} & \sum_{i=1}^m\mathcal{L}\left(\coef^\top\bm{x}_i, y_i\right)\\
    \text{s.t. }& \widehat{\mathrm{DP}}(\coef) \leq \epsilon,\label{fair.training.constr_fair}
\end{align}
\end{subequations}
where 
\begin{equation} \label{def.dp.hat} \widehat{\mathrm{DP}}(\coef) = \max_{b\in \R}\left| \frac{1}{m_1}\sum_{i=1: a_i = 1}^m \mathbbm{1}(\coef^\top\bm{x}_i > b) -\frac{1}{m} \sum_{i=1}^m \mathbbm{1}(\coef^\top\bm{x}_i > b)\right| \cdot
\end{equation}
An alternative to the fair regression problem \eqref{fair.training.constr} is the regularized version with the fairness penalty in the objective: 
\begin{equation}  \label{fair.training.regularized}
    \min_{\coef \in \R^n} \sum_{i=1}^m\mathcal{L}\left(\coef^\top\bm{x}_i, y_i\right) + \lambda \widehat{\mathrm{DP}}(\coef).
\end{equation}

In some cases, a one-sided constraint or regularization is typically sufficient, ensuring that the outcomes of the protected group are competitive. The resulting regularized version is 
\begin{equation}  \label{fair.training.oneside.regularized}
    \min_{\coef \in \R^n} \sum_{i=1}^m\mathcal{L}\left(\coef^\top\bm{x}_i, y_i\right) + \lambda \max_{b\in \R} \left\{\frac{1}{m_1}\sum_{i=1: a_i = 1}^m \mathbbm{1}(\coef^\top\bm{x}_i > b) -\frac{1}{m} \sum_{i=1}^m \mathbbm{1}(\coef^\top\bm{x}_i > b)\right\} \cdot
\end{equation}

The difficulty in solving problems \eqref{fair.training.constr}, \eqref{fair.training.regularized} and \eqref{fair.training.oneside.regularized} stems from the non-convexity and discontinuity of the linear combination of indicators used to define demographic parity. Note that \eqref{fair.training.constr_fair} can be interpreted as an infinite number of constraints, one for each possible value of $b\in \R$. For tractability, we discretize the values for $b$ as it is common in the literature \citep{agarwal2019fair} and measure fairness over a grid when optimizing fair regression problems.  Given $\ell$ thresholds, $b_1 < b_2 < \dots < b_\ell$, we define a discretized version of the distance from demographic parity:
\begin{equation} \label{def.dp.hat.ell} \widehat{\mathrm{DP}}_\ell(\coef) = \max_{j\in [\ell]}\left| \frac{1}{m_1}\sum_{i=1: a_i = 1}^m \mathbbm{1}(\coef^\top\bm{x}_i > b_j) -\frac{1}{m} \sum_{i=1}^m \mathbbm{1}(\coef^\top\bm{x}_i > b_j)\right| \cdot
\end{equation} To build fair regression models, we use the discretized relaxation of \eqref{fair.training.constr}: 
\begin{align} \label{eq:fairContinuous}
     \min_{\coef \in \R^n}  \sum_{i=1}^m\mathcal{L}\left(\coef^\top\bm{x}_i, y_i\right)  : \  \widehat{\mathrm{DP}}_\ell(\coef)\leq \epsilon. %\label{eq:fairContinuous_constr}
\end{align}
In principle, if the discretization is sufficiently granular, \eqref{eq:fairContinuous} is a close approximation of \eqref{fair.training.constr}. Similarly, we consider the discretized relaxation of \eqref{fair.training.regularized} and \eqref{fair.training.oneside.regularized} given by 
\begin{align}  
    \text{\fairabs} 	&\min_{\coef \in \R^n} \sum_{i=1}^m\mathcal{L}\left(\coef^\top\bm{x}_i, y_i\right) + \lambda \max_{j\in [\ell]} \left\{\left|\frac{1}{m_1}\sum_{i=1: a_i = 1}^m \mathbbm{1}(\coef^\top\bm{x}_i > b_j) -\frac{1}{m} \sum_{i=1}^m \mathbbm{1}(\coef^\top\bm{x}_i > b_j)\right|\right\},\label{eq:fairContinuousRegAbs} \\
    \text{\fairreg}  &\min_{\coef \in \R^n} \sum_{i=1}^m\mathcal{L}\left(\coef^\top\bm{x}_i, y_i\right) + \lambda \max_{j\in [\ell]} \left\{\frac{1}{m_1}\sum_{i=1: a_i = 1}^m \mathbbm{1}(\coef^\top\bm{x}_i > b_j) -\frac{1}{m} \sum_{i=1}^m \mathbbm{1}(\coef^\top\bm{x}_i > b_j)\right\}, 
    \label{eq:fairContinuousReg}
\end{align}
respectively. Even with a trivial loss function of zero and $\ell=1$, solving \eqref{eq:fairContinuous}--\eqref{eq:fairContinuousReg} is NP-hard, which can be shown via a reduction from maximally feasible linear subsystems \citep{amaldi1995complexity}.
% \todo{Not clear to me. If there is no loss function and $\ell=1$, is the problem really hard?}
To solve \eqref{eq:fairContinuous}--\eqref{eq:fairContinuousReg}, we propose introducing binary variables to represent the indicators involved in the fairness constraint and regularizer. 

%Observe that problems \eqref{fair.training.constr}, \eqref{fair.training.regularized}, and  \eqref{fair.training.oneside.regularized} lack lower semicontinuity, which may lead to complications \citep{han2024analysis} -- for example, problem \eqref{fair.training.oneside.regularized} may lack optimal solutions, and instead, computation of an infimum is desired. 

\subsection{Related work}
Optimizing fairness metrics such as distance to demographic parity is particularly challenging for regression because the metrics are based on the entire distribution of predictions. As such, many fairness regularizers make use of summary statistics \citep{calders2013controlling, woodworth17a, komiyama18a} to allow for efficient optimization, but they fail to fully capture the distribution of predictions, ultimately leading to less fair models \citep{oneto2020fairness}. Another approach is to build convex approximations for deviation measures of error rates across groups, such as linear approximations proposed in the works of \citet{zafar2017fairnessB} and \citet{donini2018empirical}, and non-linear but convex approximations proposed by \citet{wu2019convexity}. These approaches are predominantly designed for fair classification and additionally suffer from approximations being too loose to produce fair models, as thoroughly analyzed in \citet{lohaus2020too}. 

Research on fair regression is limited compared to classification. \citet{berk2017convex} design a convex fairness regularizer for least squares and logistic regression. \citet{do2022fair} propose a similar convex regularizer, extending results for generalized linear models and providing theoretical guarantees. Additionally, \citet{agarwal2019fair} introduce a reduction-based algorithm, transforming the fair regression training problem into a Lagrangian min-max problem, resulting in a procedure that solves a sequence of weighted classification problems. 

% Although some of the methods provide theoretical guarantees, these may not provide meaningful information in practice \citep{lohaus2020too}. An advantage of building relaxations over approximations is that they can provide bounds on the optimality gap of a given solution, allowing to certify optimality.

Recent work has utilized mixed integer optimization techniques to model and solve the training problems considering exact fairness metrics. Mixed integer formulations for fair decision trees, decision rules, and SVMs have been proposed for fair classification \citep{aghaei2019learning, jo2023learning, lawless2021fair, ye2020unbiased}. \citet{ye2024distributionally} tackle Wasserstein-fair regression by providing a fair stochastic optimization framework. This work is part of a growing research direction that uses mixed integer optimization to train models with underlying discrete structures that have been traditionally approximated with convex surrogates \citep{atamturk2019rank, bertsimas2020sparse, atamturk2020safe, hazimeh2022sparse, deza-kdir22}. 
The exact formulation of the discrete structures has been demonstrated to improve the statistical performance of the models, albeit with increased computational requirements.

\subsection{Contributions and Outline}\label{sec:contributions}
\paragraph{Contributions}

We develop novel mixed-integer optimization methods for training fair regression models. Our key contribution is a strong convex relaxation of the fair regression training problem, which can be used directly as a fair estimator, can be combined with branch-and-bound methods to solve moderately-sized problems to optimality, or can be enhanced with coordinate descent methods to produce solutions for larger datasets quickly. The strength of the proposed relaxation stems from the convexification of the discrete substructure introduced by fairness metrics and a nonlinear loss function. Through an empirical study, we demonstrate that the proposed techniques improve the out-of-sample performance of models compared to alternatives in the literature. 

Formally, given a vector $\bm{b} \in \R^\ell$ and a convex (loss) function $\mathcal{L}$, we study and describe the closed convex hull of the set 
$$X = \bigg \{(v,\bm{z}, s) \in \R \times \{0,1\}^\ell \times \R: \mathcal{L}(v) \leq s,  (v-b_j)z_j \geq 0, (b_j-v)(1-z_j) \geq 0, ~j \in [\ell] \bigg \} \cdot 
$$
%as a building block to derive a strong relaxation for the fair regression training problem.   
The set $X$ represents the epigraph of univariate loss function$\mathcal{L}$ along with binary variables indicating which interval $[b_j, b_{j+1}]$ the argument of $\mathcal{L}$ falls into. We show that $X$ can be used to model a single-factor or a single-observation fair regression as a building block for modeling and solving the general fair regression problems \eqref{eq:fairContinuous}--\eqref{eq:fairContinuousReg}.

\paragraph{Outline}

The rest of the paper is organized as follows.
In \S \ref{sec:formulations}, we provide mixed-integer formulations and strong convex relaxation for this problem. In \S \ref{sec:coordinate.descent} we provide an efficient coordinate descent method for the fair regression problem. In \S \ref{sec:comparisons}, we discuss in further detail relevant approaches to fair learning and contrast with our proposed methodology. Finally, in \S \ref{sec:experiments} we provide numerical experiments on synthetic and real datasets.

\section{MIO formulations for fair regression}\label{sec:formulations}

%The difficulty in solving problems \eqref{fair.training.constr}, \eqref{fair.training.regularized} and \eqref{fair.training.oneside.regularized} stems from the non-convexity and discontinuity of the linear combination of indicators used to define demographic parity. Note that \eqref{fair.training.constr_fair} can be interpreted as an infinite number of constraints, one for each possible value of $b\in \R$. For tractability, we discretize the values for $b$ as it is common in the literature \citep{agarwal2019fair}. Given a set of thresholds $b_1\leq b_2\leq\dots\leq b_\ell$, we consider the relaxation of \eqref{fair.training.constr} given by \begin{subequations}\label{eq:fairContinuous}
%

%\ad{Problems \eqref{eq:fairContinuous}--\eqref{eq:fairContinuousReg} are computationally hard, as disregarding the loss function and setting $\ell = 1$, we are left with linear binary classification problems with 0--1 loss, which are known to be hard \citep{feldman2012agnostic}. }
% Even with a trivial loss function of zero and $\ell=1$, solving \eqref{eq:fairContinuous}--\eqref{eq:fairContinuousReg} is NP-hard, which can be shown via a reduction from maximally feasible linear subsystems \citep{amaldi1995complexity}.
% % \todo{Not clear to me. If there is no loss function and $\ell=1$, is the problem really hard?}
% To solve \eqref{eq:fairContinuous}--\eqref{eq:fairContinuousReg}, we propose introducing binary variables to represent the indicators involved in the fairness constraint and regularizer. 
First, in \S\ref{sec:bigM} we discuss natural formulations for \fairreg\; defined in \eqref{eq:fairContinuousReg}; unfortunately, such formulations have weak convex relaxations and are ineffective in practice. 
Then, in \S\ref{sec:convexification}, we give stronger formulations based on the convexification of the single-observation subproblem to better handle the fairness constraint and regularization.

\subsection{Basic formulations}\label{sec:bigM}

By introducing auxiliary variables $v_i \in \R$ to represent $\coef^\top \bm{x_i}$, and binary variables $z_{ij} \in \{0,1\}$ as indicators for $\mathbbm{1}(\coef^\top\bm{x}_i > b_j)$,
the first MIO formulation we consider for \fairreg\; is stated as 
\begin{subequations}\label{eq:fairContinuousReg2}
\begin{align} 
 \min_{\coef,t,\bm{v},\bm{z}}\;& \sum_{i=1}^m\mathcal{L}\left(\coef^\top\bm{x}_i, y_i\right) + \lambda t\\
    \text{s.t.}\;&\frac{1}{m_1}\sum_{i=1: a_i = 1}^m z_{ij} -\frac{1}{m} \sum_{i=1}^m z_{ij}\leq t,& j\in [\ell],\\
    &v_{i}=\coef^\top \bm{x_i}, & i\in [m],\\
    &\coef\in \R^n,\;t\in \R,\; (\bm{v},\bm{z})\in V,
\end{align}
\end{subequations}
where all non-convexities and discontinuities are subsumed in the set $V$:
\begin{align*}
    V := \left\{(\bm{v},\bm{z}) \in \R^{m}\times \{0,1\}^{m\times \ell}: z_{ij}=\mathbbm{1}(v_i>b_j),~ i\in [m],~j\in [\ell] \right\}.
\end{align*}
To reformulate problems with set $V$ as mixed-integer problems, consider the set  
\begin{align*}
    \bar{V} = \Big\{(\bm{v},\bm{z}) \in \R^m\times \{0,1\}^{m\times \ell}: &~(v_i-b_j)z_{ij}\geq 0, (v_i-b_j)(1-z_{ij})\leq 0, ~i\in [m],\; j \in [\ell]\Big\}.
\end{align*}
Observe that if $(\bm{v},\bm{z})\in \bar{V}$ and $v_i>b_j$, then $z_{ij}=1$ due to constraint $(v_i-b_j)(1-z_{ij})\leq 0$; similarly, if $v_i<b_j$, then $z_{ij}=0$ due to constraint $(v_i-b_j)z_{ij}\geq 0$. However, if $v_i=b_j$, then both $z_{ij}=0$ and $z_{ij}=1$ are feasible, and in particular the equality $z_{ij}=\mathbbm{1}(v_j>b_j)$ does not hold if $z_{ij}=1$. The ambiguity when $v=b_j$ is caused by the lack of lower semi-continuity of \eqref{eq:fairContinuousReg2} and cannot be avoided.
\begin{proposition}
$\mathrm{cl}(V)=\bar{V}$.
\end{proposition}
\begin{proof}
Since $V\subseteq \bar{V}$ and $\bar{V}$ is closed (as a finite union of closed sets), the inclusion $\mathrm{cl}(V)\subseteq \bar{V}$ holds. To prove the reverse inclusion, let $(\bm{\bar v},\bm{\bar z})\in \bar{V}$, and let $S=\left\{i\in [m]: \exists j\in [\ell] \text{ s.t. }\bar v_{i}=b_{j}\text{ and }\bar z_{ij}=1\right\}$. Then letting $(\bm{v_\epsilon},\bm{z_\epsilon})$ such that $(v_\epsilon)_i=\bar{v_i}+\epsilon$ for $i\in S$, $(v_\epsilon)_i=\bar v_i$ for $i\in [m]\setminus S$ and $\bm{z_\epsilon}=\bm{\bar z}$, we find that $(\bm{v_\epsilon},\bm{z_\epsilon})\in V$ and $\lim_{\epsilon\to 0_+}(\bm{v_\epsilon},\bm{z_\epsilon})=(\bm{\bar v},\bm{\bar z})$, concluding the proof.
\end{proof}

Using the set $\bar{V}$, we obtain the mixed-integer programming formulation of \fairreg
\begin{subequations}\label{eq:fairMipReg2}
\begin{align} 
    \min_{\coef,t,\bm{v},\bm{z}}\;& \sum_{i=1}^m\mathcal{L}\left(\coef^\top\bm{x}_i, y_i\right) + \lambda t\\
    \text{s.t.}\;&\frac{1}{m_1}\sum_{i=1: a_i = 1}^m z_{ij} -\frac{1}{m} \sum_{i=1}^m z_{ij}\leq t,& j\in [\ell],\\ 
    & v_{i}=\coef^\top \bm{x_i},& i\in [m],\\
    &\coef\in \R^n,\;t\in \R,\; (\bm{v},\bm{z})\in \bar{V}.
\end{align}
\end{subequations}
Note that \eqref{eq:fairMipReg2} is a mixed-integer problem with non-convex quadratic constraints. While some off-the-shelf MIO solvers admit such formulations, our experience indicates that they are not effective in solving \eqref{eq:fairMipReg2}. A better implementation is obtained by linearizing the quadratic constraints: given a sufficiently large number $M$, the natural mixed-integer convex formulation\begin{subequations} \label{mod.bigM}
    \begin{align}
      \quad \min_{\coef, \bm{z}, t} \; & \sum_{i=1}^m \mathcal{L}(\coef^\top \bm{x}_i, y_i)+\lambda t\\
   \text{\basicmodel}  \quad      \text{s.t. }&\frac{1}{m_1}\sum_{i=1: a_i = 1}^m z_{ij} -\frac{1}{m} \sum_{i=1}^m z_{ij}\leq t, & \hspace*{-2cm}j\in [\ell],\label{bigM:fair}\\
        & \coef^\top \bm{x}_i - b_j \leq M z_{ij}, & i \in [m], j \in [\ell], \label{bigM:ub} \\
        &\coef^\top \bm{x}_i - b_j \geq -M (1-z_{ij}), & i \in [m], j \in [\ell], \label{bigM:lb} \\
        & \coef\in\R^n, \bm{z}\in \{0,1\}^{m\times \ell},t\in \R,
    \end{align}
\end{subequations}
is equivalent to \eqref{eq:fairMipReg2}. 

Unfortunately, \basicmodel\ suffers from a weak relaxation, resulting in prohibitive solution times with branch-and-bound algorithms. In fact, the lower bound obtained by solving the continuous relaxation of \basicmodel\ is, at best, equal to that of the linear regression problem that disregards fairness altogether.
To show this, consider the continuous relaxation of \basicmodel\ with constraints $\bm{z}\in \{0,1\}^{m\times \ell}$ relaxed to $\bm{0}\leq \bm{z}\leq \bm{1}$. Let $\coef^*_{\mathrm{unfair}} = \arg\min_{\coef}\sum_{i=1}^m \mathcal{L}(\coef^\top \bm{x}_i, y_i)$ be an optimal solution to the regression training problem ignoring fairness by setting $\lambda=0$ in \fairreg. Note that for large enough $M$, the solution $\left(\coef = \coef^*_{\mathrm{unfair}}, ~\bm{z} = \bm{1/2}, t=0 \right)$ is always feasible for the continuous relaxation of \basicmodel\ and has the same objective value as the linear regression problem without fairness regularization.

A natural approach to improve the relaxation is to design improved relaxations of set $\bar{V}$. Observe that this set is similar to the \textit{mixed vertex packing set} \citep{ANS:mvp} or the \emph{mixing set} \citep{gunluk2001mixing}, especially when formulated using big-$M$ constraints \eqref{bigM:ub}-\eqref{bigM:lb}, a set that is well-studied in the MIO literature. However, strong formulations based on mixed-integer linear optimization techniques exploit bounds on the continuous variable, which would require assuming that $M$ is a small number. In fact, a stronger convexification for $\bar{V}$ is not possible, as shown in the next proposition.
\begin{proposition}\label{prop:barV}
$\mathrm{cl\;conv}(\bar{V})=\Big\{ (v,\bm{z}) : \R^m\times [0,1]^{m\times \ell} : z_{i+1} \leq z_i, ~i \in [\ell-1] \Big\}$.
\end{proposition}
\begin{proof}
Note that $\bar{V}$ is the Cartesian product of $m$ sets of the form $\bar{U}= \Big\{(v,\bm{z}) \in \R\times \{0,1\}^{\ell}: ~(v-b_j)z_{j}\geq 0, (v-b_j)(1-z_{j})\leq 0, ~j \in [\ell]\Big\},$ thus is suffices to prove the result for $\bar{U}$.
First, we prove that $(v,\bm{0})\in \text{cl conv}(\bar{U})$ for all $v\in \R$. Indeed, if $v\leq b_1$, then $(v,\bm{0})\in \bar{U}$; otherwise $(v,\bm{0})=\lim_{\lambda\to 0}(1-\lambda)(b_1,\bm{0})+\lambda(b_\ell+(v-b_1)/\lambda,\bm{1})$ with both $(b_1,\bm{0})\in \bar{U}$ and $(b_\ell+(v-b_1)/\lambda,\bm{1})\in \bar{U}$ for all $\lambda\in (0,1)$. Second, note that an arbitrary point $(v,\bm{z})\in \R\times [0,1]^\ell$ such that $z_{i+1} \leq z_i, ~\forall i \in [\ell-1]$ satisfies 
\begin{align*}
    (v,\bm{z})=\lim_{\lambda\to 0}(1-\lambda)\left[z_\ell(b_\ell, \bm{1}) + \sum_{j=1}^{\ell-1} u_j b_j, \sum_{i=1}^j\bm{e_i})\right] + \lambda\left(\frac{v}{\lambda}-\frac{1-\lambda}{\lambda}\left[z_\ell b_\ell + \sum_{j=1}^{\ell-1} u_j b_j\right], \bm{0}\right),
\end{align*}
where $u_j = z_j - z_{j+1}$, $j=1, \ldots, \ell-1$,
and all points belong to $\bar{U}$, thus $(v,\bm{z})\in \text{cl conv}(\bar{U})$. 
\end{proof}
% In the next section we show how to derive stronger reformulations of fair regression problems by exploiting the objective term in the convexification.

\begin{remark}
	Formulation \basicmodel\; solves the lower-semicontinuous relaxation of \eqref{eq:fairContinuousReg} given by 
	\begin{equation*} 
	  \min_{\coef \in \R^n} \sum_{i=1}^m\mathcal{L}\left(\coef^\top\bm{x}_i, y_i\right) + \lambda \max_{j\in [\ell]} \left\{\left(\frac{1}{m_1}-\frac{1}{m}\right)\sum_{i=1: a_i = 1}^m \mathbbm{1}(\coef^\top\bm{x}_i > b_j) -\frac{1}{m} \sum_{i=1: a_i=0}^m \mathbbm{1}(\coef^\top\bm{x}_i \geq b_j)\right\},
	\end{equation*}
	as do the stronger formulations that we discuss in \S\ref{sec:convexification}. 
	Moreover, if constraints \eqref{bigM:fair} are replaced with $\left|\frac{1}{m_1}\sum_{i=1: a_i = 1}^m z_{ij} -\frac{1}{m} \sum_{i=1}^m z_{ij}\right|\leq \epsilon$ for all $j\in \ell$, then the ensuing relaxation solves a closed relaxation of \eqref{eq:fairContinuous}.
\end{remark}

\subsection{A strong convex relaxation}\label{sec:convexification}
To overcome the weak relaxations induced by the set $\bar{V}$, we propose a strong extended convex relaxation of \fairreg.
The strong relaxation is obtained by considering a more informative set that introduces the epigraph of the loss function to $\bar{V}$:$$\bar{V}_{\mathcal{L}} = \Big\{(\bm{v}, \bm{z}, \bm{s}) \in \R^m \times \{0,1\}^{m\times\ell} \times \R^m: \mathcal{L}(v_i,y_i) \leq s_i, (\bm{v}, \bm{z}) \in \bar{V}\Big\} \cdot$$
Observe that problem \fairreg\ can be rewritten as 
\begin{subequations}\label{eq:fairMipRegEpi}
\begin{align} 
    \min_{\coef,t,\bm{v},\bm{z},\bm{s}}\;& \sum_{i=1}^ms_i + \lambda t\\
    \text{s.t.}\;& \frac{1}{m_1}\sum_{i=1: a_i = 1}^m z_{ij} -\frac{1}{m} \sum_{i=1}^m z_{ij}\leq t,& j\in [\ell],\\
    &v_{i}=\coef^\top \bm{x_i},&i\in [m],\\
    &\coef\in \R^n,\;t\in \R,\; (\bm{v},\bm{z},\bm{s})\in \bar{V}_{\mathcal{L}},
\end{align}
\end{subequations}
and a relaxation \ignore{of \eqref{eq:fairMipRegEpi} }can be obtained by replacing $(\bm{v},\bm{z},\bm{s})\in \bar{V}_{\mathcal{L}}$ with $(\bm{v},\bm{z},\bm{s})\in \text{cl conv}(\bar{V}_{\mathcal{L}})$. In this section, we provide an explicit characterization of this relaxation. 

Observe that set $\bar{V}_{\mathcal{L}}$ is the Cartesian product of sets of the form
$$X_i = \{(v,\bm{z}, s) \in \R \times \{0,1\}^\ell \times \R: \mathcal{L}(v,y_i) \leq s, (v-b_j)z_j \geq 0, (b_j-v)(1-z_j) \geq 0, ~ j \in [\ell]\},$$
that is, $(\bm{v},\bm{z},\bm{s})\in \bar{V}_{\mathcal{L}}\Leftrightarrow (v_i,\bm{z_i},s_i)\in X_i, i\in [m]$. 
In Proposition~\ref{prop:valid} we derive valid inequalities for $X_i$ in an extended formulation. The formulation makes use of the perspective reformulation technique, used to model indicators on decision variables. The closure of the perspective of a proper convex function $g : \R \rightarrow \R \cup \{\infty\}$ is defined as the function $h:\R \times \R_+ \rightarrow \R \cup \{\infty\}$ such that

$$ h(u, z) = \begin{cases}
    zg(u/z) & \text{if } z > 0,\\
    \lim_{\tilde{z}\to 0} \tilde{z}g(u/\tilde{z}) & \text{if } z=0.
\end{cases}
$$
Recall that the perspective of a convex function is convex \citep{rockafellar1970convex}. And if $g(0)$ is finite, then it holds
\begin{equation}
    h(0, 0) = \lim_{\tilde{z}\to 0^+}\tilde{z}g(0/\tilde{z}) = 0. \label{eq:persp.at.zero}
\end{equation}

In Proposition \ref{prop:valid}, we present a convex extended relaxation for \eqref{eq:fairMipRegEpi}.
 Intuitively, the additional variables $\bm{p}$ we introduce model the portion of each prediction that falls in the intervals defined by the breakpoints, see Figure~\ref{fig:defP}. Specifically, $p_{ij}$ represents the amount by which prediction $\bm{w}^\top\bm{x}_i$ exceeds tis lower bound $b_j$ on the interval $[b_j, b_{j+1}]$ for $j=1, \ldots, \ell$.
For simplicity of notation, we omit the dependence on $i$, write $\mathcal{L}(v)$ instead of $\mathcal{L}(v,y_i)$, and study set $X$ defined in \S\ref{sec:contributions}, whose definition we repeat for convenience: $$X = \{(v,\bm{z}, s) \in \R \times \{0,1\}^\ell \times \R: \mathcal{L}(v) \leq s,  (v-b_j)z_j \geq 0, (b_j-v)(1-z_j) \geq 0, ~j \in [\ell]\}.$$
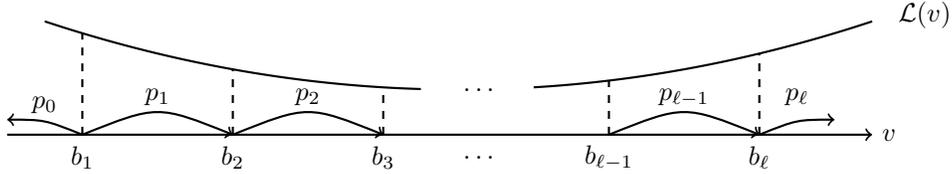
\begin{figure}[!h]
\centering
\begin{tikzpicture}[thick]
    
    % Draw axes
    \draw[->] (0, 0) -- (11.5, 0) node[right] {$v$}; % x-axis

    % Draw the convex quadratic function, stopping in the middle
    \draw[domain=0.5:5.5, smooth, variable=\x] plot ({\x}, {0.03*(\x - 6)*(\x - 6) + 0.6});
    \draw[domain=7:11.5, smooth, variable=\x] plot ({\x}, {0.03*(\x - 6)*(\x - 6) + 0.6});
    
    % Ellipsis in the middle to indicate continuation
    \node at (6.3, 0.6) {$\dots$};

    % Label points b_1, b_2, b_3, b_4 on the x-axis
    \foreach \i in {1,3,5,8,10} {
        \draw[dashed] (\i, 0) -- (\i, {0.03*(\i - 6)*(\i - 6) + 0.6});
    }

    % Labels for b_i
    \node at (1, -0.3) {$b_1$};
    \node at (3, -0.3) {$b_2$};
    \node at (5, -0.3) {$b_3$};
    \node at (8, -0.3) {$b_{\ell-1}$};
    \node at (10, -0.3) {$b_\ell$};

    % Bubbly arch arrows representing the sum of b_1 + p across boundaries
    \draw[<-, thick] (0, 0.2) .. controls (0.5, 0.2) ..  (1, 0); % from -1 to b_1
    \node at (0.5, 0.4) {$p_0$};

    % Arrows before and after ellipsis
    \draw[->, thick] (1, 0) .. controls (2, 0.4) .. (3, 0); % b_1 to b_2
    \draw[->, thick] (3, 0) .. controls (4, 0.4) .. (5, 0.0); % stops before ellipsis

    \draw[->, thick] (8, 0) .. controls (9, 0.4) .. (10, 0); % continues after ellipsis
    \draw[->, thick] (10, 0) .. controls (10.5, 0.2) .. (11, 0.2); % from b_4 to p_\ell
    \node at (10.5, 0.5) {$p_\ell$}; % Label the last point as p_\ell

    % Labels for p_1, p_2, p_3
    \node at (2, 0.5) {$p_1$};
    \node at (4, 0.5) {$p_2$};
    \node at (9, 0.5) {$p_{\ell-1}$};
    \node at (12.2, 1.6) {$\mathcal{L}(v)$};

    % Additional ellipsis on the x-axis
    \node at (6.3, -0.3) {$\dots$};

\end{tikzpicture}
\caption{Visual representation of the extended formulation of $X_i$. We introduce variables $p_j, j = 0, \ldots, \ell$, which model the piece of the prediction $v$ on the sub-intervals defined by $b_1, \ldots, b_\ell$.}
\label{fig:defP}
\end{figure}

\begin{proposition}[Validity] \label{prop:valid} Assume $\mathcal{L}(b_j)$ is finite for every $j \in [\ell]$. If $(v,\bm{z}, s)\in X$, then there exists $(p_0,\bm{p})\in \R_+^{1+\ell}$ such that 
\begin{subequations}\label{eq:validX}
\small\begin{align}
&v = b_1 - p_{0} + \sum_{j=1}^\ell p_{j},\label{eq:validX_v}\\
&(b_{j+1}-b_j)z_{j+1}\leq p_j\leq (b_{j+1}-b_j)z_j,\quad j\in [\ell],\label{eq:validX_z}\\
&s\geq (1-z_{1})\mathcal{L}\left(b_1-\frac{p_{0}}{1-z_{1}}\right) + \sum_{j=1}^{\ell-1} \left(z_{j} - z_{j+1}\right)\mathcal{L}\left(b_j + \frac{p_{j} - z_{j+1}(b_{j+1}-b_j)}{z_{j} - z_{j+1}}\right) + z_{\ell}\mathcal{L}\left(b_\ell + \frac{p_{\ell}}{z_{\ell}}\right).\label{eq:validX_L}
\end{align}\normalsize
\end{subequations}
\end{proposition}
\begin{proof}
Let $(v,\bm{z},s)\in X$. Define variables $(p_0,\bm{p})\in \R_+^\ell$ as $p_0=\max\{0,b_1-v\}$ and $p_j=\max\{0,\min\{v,b_{j+1}\}-b_j\}$. We now check the validity of \eqref{eq:validX} by cases.

\noindent $\bullet$ \textbf{Case  $\bm{v<b_1}$:} In this case, $p_0=b_1-v$ and $\bm{p}=\bm{0}$, thus constraint \eqref{eq:validX_v} is trivially satisfied. Moreover, $\bm{z}=\bm{0}$ and constraints \eqref{eq:validX_z} are also satisfied as all terms vanish. Note that, by \eqref{eq:persp.at.zero}, we have $0\mathcal{L}\left(b_j + 0/0\right)= 0$. Thus, inequality \eqref{eq:validX_L} reduces to $s\geq \mathcal{L}\left(b_1-\frac{p_{0}}{1}\right) + \sum_{j=1}^{\ell-1}0\mathcal{L}\left(b_j + 0/0\right) = \mathcal{L}\left(b_1-\frac{p_{0}}{1}\right) =\mathcal{L}\left(v\right)$, which is satisfied.\\
\noindent $\bullet$ \textbf{Case  $\bm{b_j<v<b_{j+1}}$:} In this case, $p_0=0$, $p_k=b_{k+1}-b_k$ for $k<j$, $p_j=v-b_j$ and $p_k=0$ for $k>j$, thus the telescoping sum \eqref{eq:validX_v} is satisfied. Moreover, for $k<j$, constraint \eqref{eq:validX_z} reduces (after division by $p_k=b_{k+1}-b_k$) to $z_{k+1}\leq 1\leq z_k\Leftrightarrow z_k=1$, which is implied by $(1-z_k)(b_k-v)\geq 0$; for $k>j$, constraint \eqref{eq:validX_z} reduces to $z_{k+1}\leq 0$, which is implied by $z_{k+1}(v-b_{k+1})\geq 0$; and for $k=j$, inequality $(b_{j+1}-b_j)z_{j+1}\leq p_j\leq (b_{j+1}-b_j)z_j$ is satisfied since 
both $z_{j+1}=0$ and $z_j=1$. 
Finally, \eqref{eq:validX_L} reduces to $s\geq \mathcal{L}\left(b_j + \frac{p_{j} - z_{j+1}(b_{j+1}-b_j)}{z_{j} - z_{j+1}}\right) + \sum_{k=1, k\neq j}^{\ell}0\mathcal{L}(b_k + 0/0) =\mathcal{L}(b_j+p_j)=\mathcal{L}(v)$, which is satisfied.  

\noindent $\bullet$ \textbf{Case  $\bm{v=b_{j}}$ for some $\bm{j\geq 2}$:} In this case, $p_0=0$, $p_k=b_{k+1}-b_k$ for $k<j$ and $p_j=0$ for $k\geq j$, thus the telescoping sum \eqref{eq:validX_v} is satisfied. The cases for verification of \eqref{eq:validX_z} for $k<j$ and $k>j$ are identical to the previous case, and are satisfied since $z_k=1$ for $k<j$ and $z_k=0$ for $k>j$; for $k=j$, inequality \eqref{eq:validX_z} is $(b_{j+1}-b_j)z_{j+1}\leq 0\leq (b_{j+1}-b_j)z_j$, forcing $z_{j+1}=0$ but is satisfied for any $z_j\in \{0,1\}$. 
Finally,  if $z_j=1$ then \eqref{eq:validX_L} reduces to $s\geq \mathcal{L}\left(b_j + \frac{p_{j} - z_{j+1}(b_{j+1}-b_j)}{z_{j} - z_{j+1}}\right) + \sum_{k=1, k\neq j}^{\ell}0\mathcal{L}(b_k + 0/0) =\mathcal{L}(b_j+p_j)=\mathcal{L}(b_j)=\mathcal{L}(v)$, which is satisfied; and if $z_j=0$, then \eqref{eq:validX_L} reduces to $s\geq \mathcal{L}\left(b_{j-1} + \frac{p_{j-1} - z_{j}(b_{j}-b_{j-1})}{z_{j-1} - z_{j}}\right)=\mathcal{L}(b_{j-1}+p_{j-1})  + \sum_{k=1, k\neq j-1}^{\ell}0\mathcal{L}(b_k + 0/0) =\mathcal{L}(b_j)=\mathcal{L}(v)$, which is satisfied.  

\noindent $\bullet$ \textbf{Cases  $\bm{v>b_{\ell+1}}$ and $\bm{v=b_1}$:} These cases follow from identical arguments to the previous cases and are omitted for brevity.\end{proof} 

% Note that classical regression loss functions are finite valued, and therefore satisfy the assumption in Proposition~\ref{prop:valid} \todo{We have made no explicit assumptions. Are you referring to $g(0)$ is finite? That is implicitly assumed throughout the paper}. 
Applying Proposition~\ref{prop:valid} to formulation \eqref{eq:fairMipRegEpi}, we obtain the following relaxation for fair regression.

\begin{corollary}
[Conic relaxation for fair regression] \label{prop:validity}
    The formulation
    \begin{subequations} \label{mod.strong.formulation}
    \begin{align}
        \min \; & \sum_{i=1}^m s_i + \lambda t \\
        \mathrm{s. t. } & ~\coef^\top\bm{x}_i = b_1 - p_{i0} + \sum_{j=1}^\ell p_{ij}, & i \in [m], \label{strong:equality} \\
        & (b_{j+1}-b_j)z_{ij+1} \leq p_{ij} \leq (b_{j+1}-b_j)z_{ij}, & i \in [m], j \in [\ell - 1], \label{strong:bounds} \\
    %    & 0 \leq p_{i0}, 0 \leq p_{i\ell} & \forall i \in [m]\\
        &(1-z_{i1})\mathcal{L}\left(b_1-\frac{p_{i0}}{1-z_{i1}}\right)  + z_{i\ell}\mathcal{L}\left(b_\ell + \frac{p_{i\ell}}{z_{i\ell}}\right) \notag \\
        & + \sum_{j=1}^{\ell-1} \left(z_{ij} - z_{ij+1}\right)\mathcal{L}\left(b_j + \frac{p_{ij} - z_{ij+1}(b_{j+1}-b_j)}{z_{ij} - z_{ij+1}}\right) \leq s_i,  \hspace{-1cm} & i \in [n], \label{strong:epi} \\
        &\frac{1}{m_1} \sum_{i=1: a_i = 1}^m z_{ij} - \frac{1}{m}\sum_{i=1}^m z_{ij} \leq t, & j \in [\ell],\\
        & p_{i0} \ge 0, \ p_{i\ell} \ge 0, \ \bm{z}_i \in [0,1]^\ell, & i\in[m], \label{strong:z.bounds}
    \end{align}
    \end{subequations}
    is a valid relaxation of \fairreg.
\end{corollary}

\begin{remark}
    To derive a mixed-integer reformulation of \fairreg\ from the relaxation proposed in Corollary~\ref{prop:validity}, it may not be sufficient to just enforce integrality on $\bm{z}$. Notice that big-$M$ constraints \eqref{bigM:ub} and  \eqref{bigM:lb} are not automatically implied in model \eqref{mod.strong.formulation} when we include integrality constraints, since $z_{i1} = 1$ does not immediately imply $p_{i0} = 0$, or $z_{i\ell} = 0$ does not immediately imply $p_{i\ell} = 0$. To handle this, we need to carefully consider the loss function $\mathcal{L}$. By \citet[p.259]{rockafellar1970convex}, one can show that if $\mathcal{L}$ grows superlinearly, the first two terms of constraint \eqref{strong:epi} imply $p_{i0}z_{i1} = 0$ and $p_{i\ell}(1-z_{i\ell}) = 0$ and thus no further considerations are required. Otherwise, additional big-$M$ or SOS1 constraints \citep{beale1970special} are needed.

    In particular, for the least squares regression with $\mathcal{L}(v,y_i)=(w-y_i)^2$, big-$M$ constraints are not required. For general non-linear loss functions that do not grow superlinearly, such as in the logistic regression setting where $\mathcal{Y} = \{-1,1\}$ and $\mathcal{L}(v,y_i)=\ln\left(1+e^{-y_iv}\right)$, we need to include constraints $p_{i0} \leq M(1-z_{i1})$ and $p_{i\ell} \leq Mz_{i\ell}$ to build a valid mixed-integer reformulation. Note that despite the use of big-$M$ constraints in such cases, we still enjoy an improved continuous relaxation compared to formulation \basicmodel. 
%    we indeed find that simply replacing constraint \eqref{strong:z.bounds} with integrality constraints is sufficient to build a strong reformulation of the fair training problem, owing to the quadratic loss function. However, if $\mathcal{L}$ is affine or piecewise-affine, we do not achieve any strengthening and in fact do not improve upon the relaxation introduced by set $\bar{V}$ as discussed in Proposition~\ref{prop:barV}. 
\end{remark}

We conclude this section with Proposition~\ref{prop:conv.hull}, which states that relaxation obtained in Corollary~\ref{prop:validity} is based on the closure of convex hull of $X.$ 

\begin{proposition} \label{prop:conv.hull}
    The set $\tilde{X} = \{(v, \bm{z}, p_0, \bm{p}, s) \in \R \times [0,1]^\ell \times \R^{\ell+1}_+ : \eqref{eq:validX_v}-\eqref{eq:validX_L}\}$ is an extended formulation for $\text{cl conv}(X)$. \ignore{, where \todo{how is \eqref{constr:conv.hull.eq}-\eqref{constr:conv.hull.epi} different from \eqref{eq:validX} besides non-negativity? Can we simply use \eqref{eq:validX} to avoid introducing redundant inequalities and numbering?}
    \begin{subequations}
        \begin{align}
            &v = b_1 + \sum_{i=1}^\ell p_i - p_0 \label{constr:conv.hull.eq}\\
            &(b_{i+1}-b_i) z_{i+1} \leq p_i \leq (b_{i+1}-b_i) z_i \quad \forall i = 1, \ldots, \ell-1 \label{constr:conv.hull.bounds}\\
            & 0 \leq p_0, 0 \leq p_\ell \label{constr:conv.hull.non.neg} \\
            &(1-z_1)\mathcal{L}\left(b_1 - \frac{p_0}{1-z_1}\right) + \sum_{i=1}^{\ell-1} (z_i - z_{i+1}) \mathcal{L}\left(b_i + \frac{p_i - z_{i+1}(b_{i+1}-b_i)}{z_i-z_{i+1}}\right) + z_\ell \mathcal{L}\left(b_\ell + \frac{p_\ell}{z_\ell}\right) \leq s \label{constr:conv.hull.epi}.
        \end{align}
    \end{subequations}}
\end{proposition}
\begin{proof} Let $\bm{1}_j = \sum_{k=1}^j \bm{e}_k$ be a vector where the first $j$ elements are set to one, with the rest being zero. We write $X$ as a disjunction of $\ell+1$ sets $X = \bigcup_{i=0}^{\ell} X^j,$ where 
\begin{itemize}
    \item $X^0 = \{(v^0, \bm{0}, s^0): v^0 \leq b_1, ~\mathcal{L}(v^0) \leq s^0\}$,
    \item $X^j = \{(v^j, \bm{1}_j, s^j): b_j \leq v^j \leq b_{j+1}, ~\mathcal{L}(x^j) \leq s^j \}, \; j = 1, \ldots, \ell-1$,
    \item $X^\ell = \{(v^\ell, \bm{1}, s^\ell): b_\ell \leq v^\ell, ~\mathcal{L}(v^\ell) \leq s^\ell\}$.
\end{itemize}

\noindent Then $(v, \bm{z}, s) \in \mathrm{conv}(X)$ if and only if the following system is feasible
\begin{subequations}
    \begin{align}
    &v = \sum_{j=0}^\ell \alpha_j v^j, ~s = \sum_{j=0}^\ell \alpha_j s^j,\\
    &\bm{z} = \sum_{j=1}^\ell \alpha_j \bm{1}_j, ~ \bm{1}^\top \bm{\alpha} = 1, ~ \bm{\alpha} \geq 0 \label{eq:lambda} \\
    &v^0 \leq b_1, v^\ell \geq b_\ell,\\
    &b_j \leq v^j \leq b_{j+1}, \quad ~j \in [\ell-1], \\
    &\mathcal{L}(v^j) \leq s^j, \quad \quad  \quad j = 0, \ldots, \ell.
    \end{align}
\end{subequations}
We project out $\bm{\alpha}$ as constraints \eqref{eq:lambda} imply $\alpha_0 = 1-z_1$, $\alpha_j = z_j - z_{j+1}, ~j = 1, \ldots, \ell-1$ and $\alpha_\ell = z_\ell$ and obtain the equivalent system
\begin{subequations}
    \begin{align}
    &v = (1-z_1)v^0 + \sum_{j=1}^\ell (z_j-z_{j+1}) v^i + z_\ell v^\ell,\\
    & (1-z_1)\mathcal{L}(v^0) + \sum_{j=1}^{\ell-1} (z_j - z_{j+1}) \mathcal{L}(v^j) + z_\ell \mathcal{L}(v^\ell) \leq s, \label{eq.delta.ell} \\
    &v^0 \leq b_1, v^\ell \geq b_\ell,\\
    &b_j \leq v^j \leq b_{j+1}, ~ j \in [\ell-1], \\
    & \bm{z} \in [0,1]^\ell.
    \end{align}
\end{subequations}
Introducing $p_0 = b_1-v^0, p^j = v^j - b_j, ~j \in [\ell],$ equivalent to shifted $v^i$, we rewrite the system 
\begin{subequations}
    \begin{align}
    &v = b_1 - (1-z_1)p^0 + \sum_{i=1}^{\ell-1}\left[(z_j-z_{j+1})p^j + z_{j+1}(b_{j+1}-b_j) \right] + z_\ell p^\ell, \\
    & (1-z_1)\mathcal{L}(b_1-p^0) + \sum_{j=1}^{\ell-1} (z_j - z_{i+j}) \mathcal{L}(b_j + p^j) + z_\ell \mathcal{L}(b_\ell + p^\ell) \leq s,\\
    &0 \leq p^0, 0 \leq p^\ell\\,
    &0 \leq p^j \leq b_{j+1}-b_j, ~j \in [\ell-1], \\
    & \bm{z} \in [0,1]^\ell.
    \end{align}
\end{subequations}
A final change of variables $p_0 = (1-z_1)p^0$, $p_\ell = z_\ell p^\ell$, and $p_j = (z_j-z_{j+1})p^j + z_{j+1}(b_{j+1}-b_j)$ recovers the set $\tilde{X}$.
\end{proof}

Proposition~\ref{prop:conv.hull} implies that the strong convex relaxation \eqref{mod.strong.formulation} is exact for the single-observation fair training problem that arises when $m=1$. Moreover, it is also exact for the single-factor problem when $n=1$, which directly follows from the formulation of \fairreg. To see that single-factor fair regression training reduces to optimization over a single set of type $X$, assuming $x_i \neq 0, ~i \in [m]$, we can formulate the problem as follows
\begin{subequations}
    \begin{align}
    \min_{w, \bm{z}} \; & s + \lambda t\\
    \text{s.t. } &\frac{1}{m_1}\sum_{i=1:a_i=1}^m z_j - \frac{1}{m}\sum_{i=1}^m \leq t, \hspace{3.15cm}j \in [\ell], \label{eq:single-factor-fairness}\\
    & \left(w-b_j/x_i\right)z_{ij} \geq 0, (b_j/x_i-w)(1-z_{ij}) \geq 0, \hspace{0.5cm}i \in [m], j \in [\ell],\\
    & \sum_{i=1}^m \mathcal{L}(x_iw) \leq s\\
    & w \in \R, \bm{z} \in \{0,1\}^{m\times\ell} \label{eq:single-factor-variables}.
    \end{align}
\end{subequations}
 Constraints \eqref{eq:single-factor-fairness}-\eqref{eq:single-factor-variables} can be represented by a single set of the form of $X$. %For quadratic loss functions, $\text{cl conv}(X)$ is conic-quadratic representable, implying that the single-factor least squares fair training problem is polynomial-time solvable. 

\section{Coordinate descent algorithm} \label{sec:coordinate.descent}
%For large problem instances,\todo{need to state that we consider abs value} solving the exact problem via branch-and-bound may be computationally prohibitively. 
To complement the relaxations and mixed-integer optimization formulations introduced in \S\ref{sec:formulations}, we propose a coordinate descent method tailored for the regularized fair regression problem that improves upon a given initial solution. Coordinate descent methods iteratively update one coordinate at a time by minimizing the objective function of the problem with respect to that coordinate while keeping all others fixed. The proposed methods can be used to tackle either \eqref{eq:fairContinuousRegAbs} or \eqref{eq:fairContinuousReg} (with simple modifications); here, we present the discussion for problem \eqref{eq:fairContinuousRegAbs}.

Suppose after $t$ iterations, we are given a model with regression coefficients $\coef^t$. At iteration $t+1$, coordinate $k$ is updated by solving the following problem with all other coordinates $p\neq k$ remaining fixed:
\begin{equation}\label{cd.update}
        w^{t+1}_k \in \arg\min_{w_k\in\R}  F(w_k) :=\sum_{t=1}^m \mathcal{L}(w_kx_{ik} + \sum_{p \neq k} w_p^tx_{ip}) +\lambda R(w_k),
\end{equation}
where $R(w_k)$ is the fairness regularizer
\begin{subequations} \label{eq:R(w)}
    \begin{align}
        R(w_k) = \min_{\bm{z}} & \max_{j \in [\ell]}\left|\frac{1}{m_1}\sum_{i=1: a_i=1}^m z_{ij}-\frac{1}{m}\sum_{i=1}^mz_{ij}\right| \\
        \text{s.t. }& (w_kx_{ik} + \sum_{p \neq k} w_p^tx_{ip} - b_j)z_{ij}\geq 0, & i \in [m],  j \in [\ell] ,\label{constr:cd.nonneg}\\
        & (w_kx_{ik} + \sum_{p \neq k} w_p^tx_{ip} - b_j)(1-z_{ij}) \leq 0, & i \in [m], j \in [\ell], \label{constr:cd.nonpos}\\
        &\bm{z}_{ij} \in \{0,1\}^{m\times\ell}.
    \end{align}
\end{subequations}
To solve \eqref{cd.update} efficiently, we use Proposition~\ref{prop:cd} which states that the solution $w_k^{t+1}$ is found by either ignoring fairness, or is equal to $\tilde{b}_{ij} = (b_j- \sum_{p \neq k} w_p^tx_{ip} )/x_{ij}$ for some $i\in[m], j\in[\ell]$. We then obtain a simple coordinate descent algorithm: given a solution $\coef^t$ and coordinate direction $k$, first compute $\tilde{b}_{ij}, i \in[m], j \in [\ell]$, and set  $w_k^{t+1}$ to the best solution among these points and the unconstrained solution that ignores fairness.
% Are we assuming a unique minimizer?
\begin{proposition}\label{prop:cd} 
    Let $w_{\mathrm{nofair}} = \arg\min_{w_k}\sum_{i=1}^m \mathcal{L}\left(w_kx_{ik} + \sum_{p \neq k} w_p^tx_{ip}\right)$ be the minimizer of the loss function along direction $k$ with no fairness regularization. Then, there exists a minimizer $w^{t+1}_k$ to \eqref{cd.update} such that:
    \begin{equation*}
        w^{t+1} = \arg\min_{w_k \in \tilde{B} \cup \{w_\mathrm{nofair}\}} F(w_k),
    \end{equation*}
    where $\tilde{B} = \left\{\tilde{b}_{ij} =\left(b_j - \sum_{p \neq k} w_p^tx_{ip} \right)/x_{ik}: i  \in [m], ~ j \in [\ell], x_{ij} \neq 0 \right\}$.
\end{proposition}

\begin{proof}
Let $\tilde{b}_{[s]}$ be the $s$-th smallest element of $\tilde{B}$ and define a partition of the real line, $\tilde{b}_{[1]}  <  \tilde{b}_{[2]} < \ldots <\tilde{b}_{[|\tilde{B}|]}$. The minimizer of \eqref{cd.update} must also be a minimizer for one of the finite sub-intervals defined by the partition. Let $w^*_{ks}$ by the best solution over sub-interval $\left[\tilde{b}_{[s]},\tilde{b}_{[s+1]}\right]$:
\begin{equation}
    w^*_{ks} = \arg \min_{w_k \ \in \left[\tilde{b}_{[s]}, \tilde{b}_{[s+1]}\right]} F(w_k). \label{opt.single.sub.interval}
\end{equation}
Then solving \eqref{cd.update} is equivalent to 
\begin{equation}
    w^{t+1}_k = \arg\min_{w^*_{ks}: s \in [|\tilde{B}|]} F(w^*_{ks}). \label{opt.all.sub.intervals}
\end{equation}
To solve \eqref{opt.single.sub.interval}, we first rewrite the constraints \eqref{constr:cd.nonneg} and \eqref{constr:cd.nonpos} in terms of $\tilde{b}_{ij}$:
\begin{subequations}
    \begin{align}
        (w_k-\tilde{b}_{ij})x_{ik}z_{ij} &\geq 0,  ~i\in[m],j\in[\ell], \label{cd.proof:nonneg}\\
        (w_k-\tilde{b}_{ij})x_{ik}(1-z_{ij}) &\leq 0, ~i\in[m],j\in[\ell].\label{cd.proof:nonpos}
    \end{align}
\end{subequations}
Given $w_k \in \left(\tilde{b}_{[s]}, \tilde{b}_{[s+1]}\right)$, the sign of $(w_k - \tilde{b}_{ij})$ is constant and nonzero and thus by constraints \eqref{cd.proof:nonneg} and \eqref{cd.proof:nonpos}, there are unique values $z_{ij}\in \{0,1\}$ feasible for \eqref{eq:R(w)}. Therefore, $R(w_k)$ is equal to some constant $c_s$ in the interior of each sub-interval. A similar argument is made for the open intervals $(-\infty, \tilde{b}_{[1]})$ and $(\tilde{b}_{|\tilde{B}|}, \infty)$.

Now consider $R(w_k)$ evaluated at a boundary points $\tilde{b}_{[s]}$. When $w_k  = \tilde{b}_{[s]}$, $\bm{z}$ is no longer completely determined by $w_k$ as $z_{ij}$ associated with $\tilde{b}_{[s]}$ is free to be zero or one. The additional freedom in $z_{ij}$ implies that $R(\tilde{b}_{[s]}) \leq c_s$. Similarly, we have that $R(\tilde{b}_{[s+1]}) \leq c_s$, implying that the minimum of $R$ on the sub-interval is on the boundary points. Finally, by convexity, the minimum of $\mathcal{L}$ on sub-interval $\left[\tilde{b}_{[s]}, \tilde{b}_{[s+1]}\right]$ must be at one of the boundary points, or $w_\mathrm{nofair}$ if the global minimizer of $\mathcal{L}$ is in the sub-interval. Thus, we have $$\min_{w_k \ \in \left[\tilde{b}_{[s]}, \tilde{b}_{[s+1]}\right]} F(w_k) =\begin{cases}
    \min\left(F\left(\tilde{b}_{[s]}\right), F\left(\tilde{b}_{[s+1]}\right)\right), & \text{if } w_{\mathrm{nofair}} \not\in \left(\tilde{b}_{[s]}, \tilde{b}_{[s+1]}\right)\\
    \min\left(F\left(\tilde{b}_{[s]}\right), F\left(\tilde{b}_{[s+1]}\right), F\left(w_\mathrm{nofair} \right)\right), & \text{else.}
\end{cases}$$
Therefore, it suffices to evaluate $F$ at each boundary point $\tilde{b}$, as well as $w_{\mathrm{nofair}}$ to solve \eqref{opt.all.sub.intervals}.
\end{proof}
\begin{algorithm}[h]
\caption{Coordinate descent for fair regression.}\label{alg:cd}
\begin{algorithmic}
\Require $\{\bm{x}_i, y_i, a_i\}_{i=1}^m$ data; $\{b_j\}_{j=1}^\ell$ breakpoints; $\lambda$  regularization strength; $\coef^0$ initial solution;
\Ensure $\coef^\mathrm{CD}$ coordinate-wise local optimal solution
\State $t \leftarrow 0$
\While{$\coef^t$ is not coordinate-wise locally optimal}
\State Select coordinate $k \in [n]$
\State $\tilde{B} \leftarrow \left\{ (b_{ij} - \sum_{p=1,p\neq k}^n w^t_{p} x_{ip})/x_{ik}: i \in [n], j \in [\ell], x_{ij} \neq 0 \right\}$
\State $w_{\mathrm{nofair}} \leftarrow \arg\min_{w_k}\sum_{i=1}^m \mathcal{L}\left(w_kx_{ik} + \sum_{j=1:j\neq p}^n w^t_p x_{ip} \right)$
\State $w^{t+1}_k \leftarrow \arg\min_{w_k \in \tilde{B} \cup \{w_{\mathrm{nofair}}\}} F(w_k)$
\State $t \leftarrow t + 1$
\EndWhile
\State $\coef^{\mathrm{CD}} \leftarrow \coef^t$
\end{algorithmic}
\end{algorithm} 
\noindent Each iteration of Algorithm~\ref{alg:cd} requires one oracle call to obtain the solution to the univariate vanilla regression problem to obtain $\bm{w}_{\mathrm{nofair}}$, and $\mathcal{O}(m \times \ell)$ operations to evaluate $F$ at every point in $\tilde{B}$. 

% \begin{remark}
%     The optimization problem \eqref{cd.update} solved at each iteration of Algorithm~\ref{alg:cd} can be reformulated as optimizing a linear function over $X$, with breakpoints $\tilde{b}$. Since Proposition~\ref{prop:conv.hull} provides a convex hull representation for $X$, we can alternatively build and solve a convex problem at every iteration of Algorithm~\ref{alg:cd}.\todo{Do we? We have a convex hull just for the linear case, not the absolute value case. I think it is true if we solve two optimization problems. By the way, not important at the moment, but it almost looks like we might benefit from doing so.}
% \end{remark}
% Discuss $U_k$ easy to find for linear regression, but approximated for logistic regression. 

% Although, we show that in numerical results that the convex relaxation provides competitive models, as $\epsilon$ becomes very small in \eqref{mod.strong.formulation}, or $\lambda$ very large in the regularized analogue, the convex relaxation does not provide fairer solutions. A known limitation of convex approximations for fair classification, is their difficulty to give very fair solutions. In practice, models with perfect or close to perfect fairness may be undesirable due to high inaccuracy, but constructing such solutions is important for generating Pareto-frontier curves. 

\section{Algorithms, comparisons and extensions}\label{sec:comparisons}

Using relaxations, mixed-integer optimization, and coordinate descent methods discussed in \S\ref{sec:formulations} and \ref{sec:coordinate.descent}, we propose three methodologies that balance solution quality and run time. First, in \S\ref{sec:algOverview}, we present an overview of these three approaches. Next, in \S\ref{sec:fairClassification}, we discuss the special case with $\ell=1$, which has received more attention in the literature, and compare it with existing approaches. Finally, in \S\ref{sec:extensions}, we discuss additional extensions of the proposed methods.

\subsection{Overview of proposed algorithms}\label{sec:algOverview}
The choice of the appropriate method depends on the size of the training data and the granularity of the discretization. All three methods leverage the convex hull of $X$, as presented in Proposition~\ref{prop:conv.hull}, as a fundamental tool. Below, we summarize the methods in order of increasing computational cost.

\subsubsection{Convex relaxation}\label{sec:algConvexRelax} This approach calls for solving the convex relaxation \eqref{mod.strong.formulation} to optimality and directly using regression coefficients $\bm{w^*}$ as the estimator. Since this method only requires solving a single convex relaxation, it can be efficiently handled using off-the-shelf solvers (e.g., solving instances with $m=2,000$ observations in about five seconds), and it is particularly effective at promoting fairness. In our experiments with real data presented in \S\ref{sec:computationsReal}, we demonstrate that it is an order of magnitude faster than alternative methods in the literature, while producing estimators with comparable or superior quality. 

\subsubsection{Convex relaxation + coordinate descent}\label{sec:algCoordinateDescent} 
While the convex relaxation methods discussed in \S\ref{sec:algConvexRelax} are effective in promoting fairness, they may struggle to produce estimators $\bm{w}$ satisfying statistical parity, that is, where $\widehat{\mathrm{DP}}(\coef)\approx 0$ as the convex relaxation tends to underestimate  $\widehat{\mathrm{DP}}$. In such cases, coordinate descent methods (Algorithm \ref{alg:cd}) may offer promising alternative. Although computationally more expensive than solving a single convex relaxation, coordinate descent methods scale efficiently
to problems with thousands of data points and exactly evaluate the value of $\widehat{\text{DP}}$ in the objective. However, the quality of the estimator depends on the initial point. In our experiments in \S\ref{sec:computationsCoordinate}, we show that initializing coordinate descent with the optimal solution of the strong convex relaxation \eqref{mod.strong.formulation} produces better in-sample estimators than using other natural starting points. Moreover, in our experiments with real data in \S\ref{sec:computationsReal}, we find that coordinate descent methods initialized with the solution of the convex relaxation tend to result in better estimators out-of-sample as well. Coordinate descent is able to produce models with very high levels of fairness, which the standalone convex relaxation may fail to accomplish. Given that the computational cost of solving the convex relaxation \eqref{mod.strong.formulation} is negligible relative to the overall cost of coordinate descent, we recommend solving the relaxation first to provide a strong starting point. 

\subsubsection{Mixed-integer optimization}\label{sec:algMIP} 
The third and the most computationally intensive approach is solving a mixed-integer optimization problem by adding integrality constraints and potentially big-$M$ constraints to formulation \eqref{mod.strong.formulation}. Our experiments in \S\ref{sec:computationsMIO} show that this method is much faster and effective in reducing the optimality gaps when compared with mixed-integer optimization methods based on natural big-$M$ formulation \basicmodel. Moreover, in experiments with small synthetic instances, we observe that exact methods produce better estimators (both in-sample and out-of-sample) than alternatives based on relaxations or coordinate descent, particularly in settings where fairness is more important than accuracy. However, despite the advancement presented in this paper, solving instances with $m>100$ remains a daunting task. The faster approaches discussed in \S\ref{sec:algConvexRelax} and \S\ref{sec:algCoordinateDescent} are currently more practical and preferable for large instances.

\subsection{Fair classification metrics}\label{sec:fairClassification}

\label{para:convex.approx} Most techniques proposed for training fair models have been focused on optimizing under constraints defined by fair \emph{classification} metrics, that is, $\widehat{\mathrm{DP}}_\ell$ with $\ell=1$. In this context, empirical distance from demographic parity simplifies to 
\begin{equation}\label{binary.class.fairness}
     \widehat{\mathrm{DP}}_1= \left | \frac{1}{m_1}\sum_{i=1:a_i=1}^m\mathbbm{1}(\coef^\top\bm{x}_i > 0)  - \frac{1}{m_0}\sum_{i=1:a_i=0}^m \mathbbm{1}(\coef^\top\bm{x}_i > 0)  \right |,
\end{equation}
where we assume $b_1=0$ for simplicity. %Our proposed models, therefore, apply to this setting as well. 
In this section, we compare the proposed strong extended relaxation with convex proxy methods from the literature designed for $\widehat{\mathrm{DP}}_1$.

\subsubsection{Comparison with the literature}\label{sec:fairClassificationLiterature}
A popular approach in the literature to incorporate $\widehat{\mathrm{DP}}_1$ into a fair regression problem is to design a convex approximation of the indicator function used in \eqref{binary.class.fairness} and obtain a model by solving a proxy convex optimization problem.
\citet{zafar2017fairnessB} consider approximating  $\mathbbm{1}(\coef^\top\bm{x_i} > 0)$ with the linear term itself, $\coef^\top\bm{x_i}$, giving rise to the following approximation of $\widehat{\mathrm{DP}}_1$ \eqref{binary.class.fairness}:

\begin{equation*}\label{dp.approx.linear}
\widehat{\mathrm{DP}}_1^\mathrm{linear} = \left | \frac{1}{m_0}\sum_{i=1:a_i=0}^m \coef^\top \bm{x}_i - \frac{1}{m_1}\sum_{i=1:a_i=1}^m \coef^\top \bm{x}_i  \right | \cdot
\end{equation*}
As pointed out by \citep{olfat2018spectral}, this approximation is closely related to the relaxation of the formulation \basicmodel. Consider the system of inequalities:
\begin{subequations}\label{system.zafar}
    \begin{align} 
        &-M(1-z_i) \leq \coef^\top\bm{x}_i \leq Mz_i, \; i \in [m],\\
        & \left|\frac{1}{m_1}\sum_{i=1:a_i=1}^mz_i - \frac{1}{m_0}\sum_{i=1:a_i=0}^m z_i\right| \leq t.
    \end{align}
    \end{subequations}
Rearranging the big-$M$ constraints as $\frac{1}{M}\coef^\top\bm{x}_i \leq z_i \leq \frac{1}{M}\coef^\top\bm{x}_i + 1$, we see that the above system of inequalities simply reduces to the following condition:
$$ \left|\frac{1}{m_1}\sum_{i=1:a_i=1}^m \coef^\top \bm{x}_i - \frac{1}{m_0}\sum_{i=1: a_i=0}^m \coef^\top \bm{x}_i\right|  \leq \tilde{t},$$
where $\tilde{t} = M(t-1)$ and the approximation of \citet{zafar2017fairnessB} is recovered.

The convex approximation proposed by \cite{wu2019convexity} uses convex and concave surrogate functions to approximate the indicators for both sides of the absolute function separately. 
%\ad{To derive this approximation, $\widehat{\mathrm{DP}}_1$ is first rewritten as: \begin{align*}
    % \widehat{\mathrm{DP}}_1  = \min&\; t \\
    % \text{s.t. } & \frac{1}{m_1}\sum_{i=1:a_i=1}^m \mathbbm{1}(\coef^\top\bm{x}_i>0) - \frac{1}{m_0}\sum_{i=1:a_i=0}^m \left(1-\mathbbm{1}(-\coef^\top\bm{x}_i \geq 0) \right) \leq t, \\
    % &-\frac{1}{m_1}\sum_{i=1:a_i=1}^m \mathbbm{1}(\coef^\top\bm{x}_i>0) + \frac{1}{m_0}\sum_{i=1:a_i=0}^m \left(1-\mathbbm{1}(-\coef^\top\bm{x}_i \geq 0) \right) \leq t.\\ 
    %\end{align*} 
Given $\kappa(v)$  and $\delta(v)$, convex surrogates for the indicator functions $\mathbbm{1}(v>0)$ and $-\mathbbm{1}(v>0)$ respectively, the non-linear convex approximation of $\widehat{\mathrm{DP}}_1$ \eqref{binary.class.fairness} is given by:
\begin{align*}
    \widehat{\mathrm{DP}}_1^\mathrm{convex} = \min\; t  \ 
    \text{ s.t. } & \frac{1}{m_1}\sum_{i=1: a_i=1}^m \kappa(\coef^\top\bm{x}_i) + \frac{1}{m_0}\sum_{i=1: a_1 = 0}^m \kappa(-\coef^\top\bm{x}_i) - 1 \leq t, \\
    & \frac{1}{m_1}\sum_{i=1: a_i=1}^m \delta(\coef^\top\bm{x}_i) + \frac{1}{m_0}\sum_{i=1: a_1 = 0}^m \delta(-\coef^\top\bm{x}_i) + 1 \leq t.
\end{align*} 
Hinge functions $\kappa(v) = \max(0,v+1)$ and $\delta(v)=\min(1,v)$ are commonly utilized surrogate functions \citet{wu2019convexity}. One can also connect this technique to formulation \basicmodel~by adding bound constraints $\bm{0}\leq \bm{z} \leq \bm{1}$ to the system of inequalities \eqref{system.zafar}. Combining big-$M$ constraints and variable bounds, one can rewrite the conditions on $\bm{z}$ as
$$\max(0,\frac{1}{M}\coef^\top\bm{x}_i) \leq z_i \leq \min(0, \frac{1}{M}\coef^\top\bm{x}_i) + 1.$$
By using the transformation $\max(0,x) = -\min(0,-x)$ and the above tighter lower and upper bounds reveals that the system of equations \eqref{system.zafar} with the additional bound constraints $\bm{0}\leq\bm{z}\leq\bm{1}$ is equivalent to
$$\frac{1}{m_1}\sum_{i=1: a_i=1}^m \frac{1}{M}\max(0, \coef^\top\bm{x}_i) + \frac{1}{m_0}\sum_{i=1: a_i=0}^m \frac{1}{M}\max(0,-\coef^\top\bm{x}_i) -1 \leq t.$$
Notice that we can also write the above in terms of $\min$ functions
$$-\frac{1}{m_1}\sum_{i=1: a_i=1}^m \frac{1}{M}\min(0, -\coef^\top\bm{x}_i) + \frac{1}{m_0}\sum_{i=1: a_i=0}^m \frac{1}{M}\min(0,\coef^\top\bm{x}_i) + 1\leq t.$$
The above analysis shows that continuous relaxation of \basicmodel\ is equivalent to the approximation proposed by \citet{wu2019convexity} with a shifted and scaled hinge function used as a convex surrogate for the indicator: $\kappa(v) = \frac{1}{M}\max(0, v)$ and $\delta(x) = \frac{1}{M}\min(0, -x)$.
% \begin{align*}
%     &\frac{1}{m_1}\sum_{i=1: a_i=1}^m \max(0, \coef^\top\bm{x}_i) + \frac{1}{m_1}\sum_{i=1: a_i=1}^m\max(0,-\coef^\top\bm{x}_i)-1 \leq M(\epsilon+1)-1,\\
%     &\frac{1}{m_1}\sum_{i=1: a_i=1}^m \min(0, -\coef^\top\bm{x}_i) + \frac{1}{m_1}\sum_{i=1: a_i=1}^m\min(0,\coef^\top\bm{x}_i)-1 \geq -M(\epsilon+1) + 1,
% \end{align*}

% \noindent which is equivalent to the approximation proposed by \citet{wu2019convexity}, with the shifted hinge function used as a convex surrogate: $\kappa(x) = \max(0, x)$ and $\delta(x) = \min(0, -x)$.

The approximations $\widehat{\mathrm{DP}}_1^\mathrm{linear}$ and $\widehat{\mathrm{DP}}_1^\mathrm{convex}$ are similar to continuous relaxations of \basicmodel, where instead of using a large enough constant for a valid (but ultimately weak) relaxation, $M$ is shrunk and used as a hyperparameter to promote fairness. These approaches are in contrast with the approach discussed in \S\ref{sec:algConvexRelax}, which is based on provable (and stronger) relaxations. Crucially, the relaxation \eqref{mod.strong.formulation} exploits the non-linear objective, allowing for a convexification that considers fairness and accuracy simultaneously.

We illustrate this difference by considering a two-dimensional toy dataset generated according to the procedure outlined in \citet{zafar2017fairnessB}, which we detail in Appendix \ref{appendix.data}. Figure~\ref{fig:approx} shows the objective of the regularized fair logistic regression problem for different convex approximations as $w_1$ varies,
while the second coordinate is fixed, $w_2 = 0$. We visually compare the true objective with the objectives of \citet{zafar2017fairnessB}, \citet{wu2019convexity}, and the proposed convex relaxation \eqref{mod.strong.formulation} as the weight for the fairness regularization, $\lambda$, increases. For the convex approximation of \citet{wu2019convexity}, the hinge function is used as a surrogate for the indicator, as suggested in their paper.

\begin{figure}[h]
    \centering
    \includegraphics[width=0.48\textwidth]{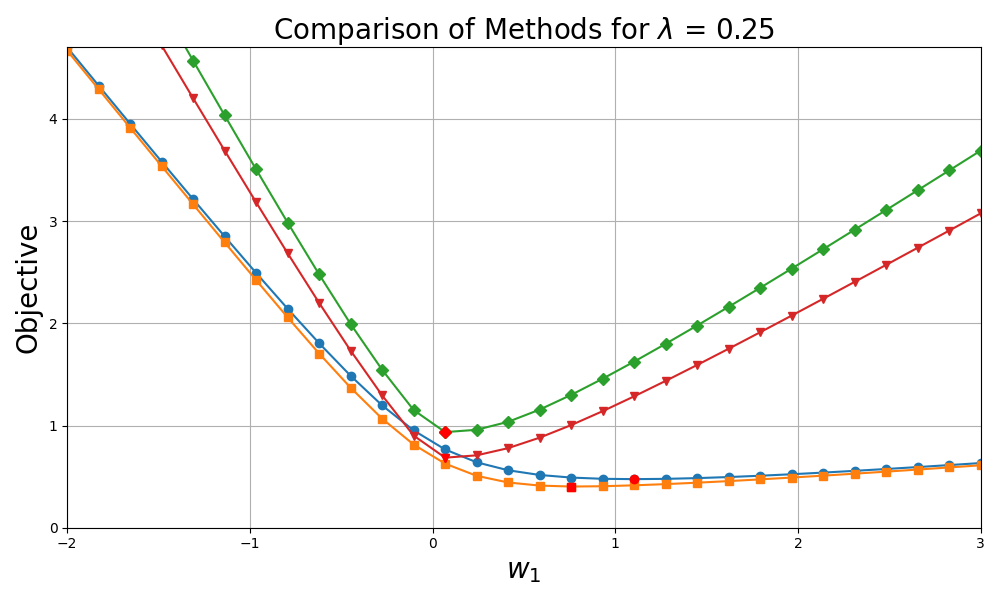}
    \hfill
    \includegraphics[width=0.48\textwidth]{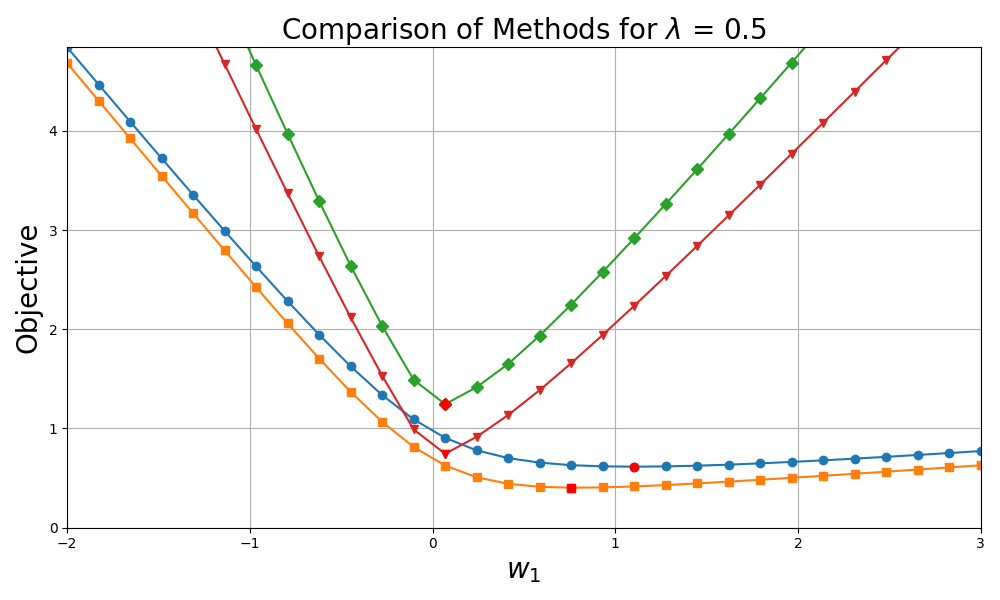}

    \includegraphics[width=0.48\textwidth]{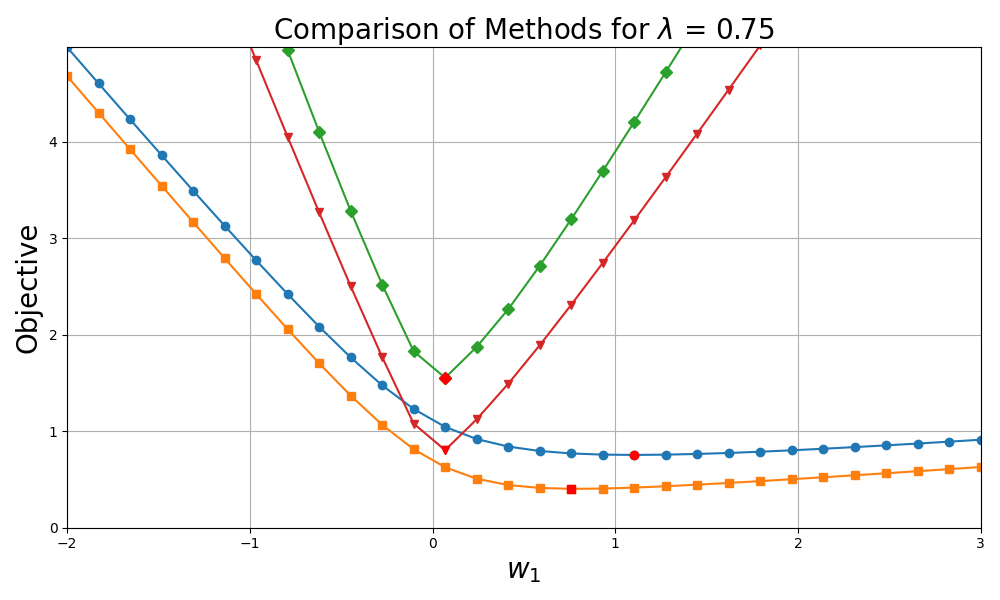}
    \hfill
    \includegraphics[width=0.48\textwidth]{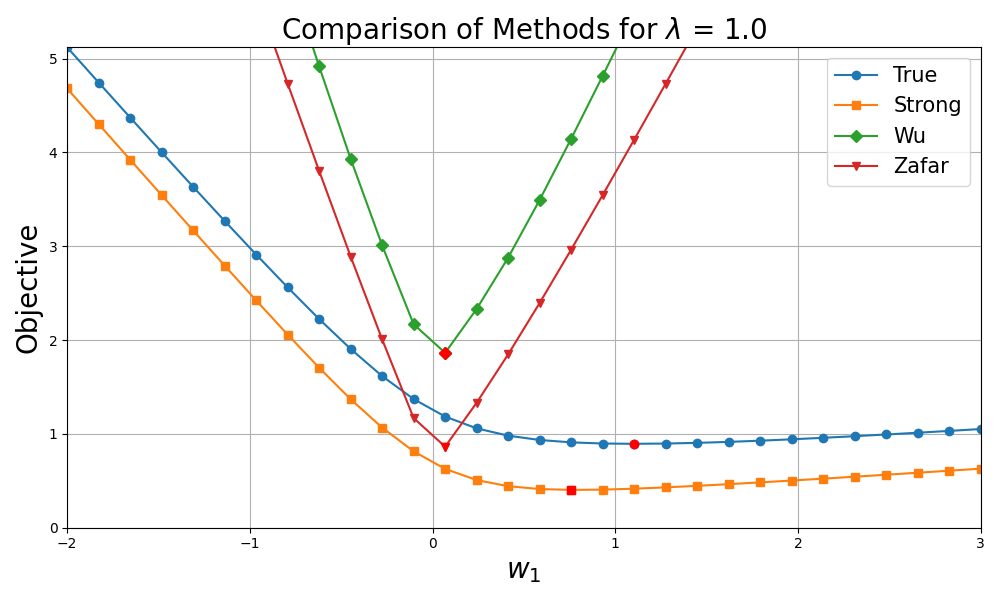}
    
    \caption{Regularized fair logistic regression objective for the true problem, strong convex relaxation~\eqref{mod.strong.formulation}, and convex approximations of \cite{zafar2017fairnessB} and \cite{wu2019convexity} for varying regularization weights ($\lambda$).}
    \label{fig:approx}
\end{figure}

 Figure~\ref{fig:approx} shows that the strong relaxation follows the shape of the true objective more closely than the convex approximations. The minimizer of the \citet{zafar2017fairnessB} and \citet{wu2019convexity} approximations is close to the trivially fair and uninformative solution obtained by $\coef = \bm{0}$. Meanwhile, the minimizer for the strong convex relaxation is able to achieve a better trade-off between fairness and accuracy for all regularization strengths. Note that since the plots consider only a single variable, if we were to solve the one-sided fairness problem \eqref{fair.training.oneside.regularized}, then by Proposition \ref{prop:conv.hull}, the strong relaxation would be ideal, thus exactly matching the true objective. 

 \subsubsection{Improved method for out-of-sample performance}\label{sec:fairClassificationApproximation}

In our experiments, we observed that imposing constraints such as $\widehat{\mathrm{DP}}_1\leq t$ can lead to solutions with poor out-of-sample performance. The resulting models, obtained by solving the non-convex problems exactly or via convex relaxations/coordinate descent, tend to produce models where $|\bm{w^\top x_i}|$ is close to $0$ for multiple data points $i\in [m]$. Consequently, even small perturbations in the input vector $\bm{x_i}$ can cause the sign of the prediction $\bm{w^\top x_i}$ to flip. This means seemingly fair estimators may actually violate statistical parity under slight perturbations, leading to poor out-of-sample performance. Interestingly, the convex proxies proposed by \citet{zafar2017fairnessB} and \citet{wu2019convexity} implicitly address this phenomenon. By approximating the discrete fairness constraints, these methods introduce larger penalizations or rewards as the distance of predictions from $0$. This results in less ambiguous estimators in-sample and better out-of-sample performance. 

The generalization properties of estimators based on good approximations of constraint $\widehat{\mathrm{DP}}_1\leq t$ can be further improved with a simple yet effective trick: even when fairness with respect to a single threshold $b=0$ is desired, we train the model with respect to a number of equally-spaced thresholds $\{-b_{\lfloor\ell/2\rfloor}, -b_{\lfloor\ell/2\rfloor-1}, -b_1, 0, b_1,\dots, b_{\lfloor\ell/2\rfloor-1}, b_{\lfloor\ell/2\rfloor}\}$. Introducing these artificial thresholds encourages estimators to maintain a larger margin around the threshold $0$. Our computational experiments in \S\ref{sec:compFairLogistic} show that applying this trick when solving the convex relaxation \eqref{mod.strong.formulation} results in estimators with superior out-of-sample performance compared to the convex proxies of \citet{zafar2017fairnessB} and \citet{wu2019convexity}.

\subsection{Extensions}\label{sec:extensions}

Thanks to the modeling power of mixed-integer optimization, various modifications to problems \eqref{eq:fairContinuous}-\eqref{eq:fairContinuousReg} can be seamlessly incorporated. A natural extension is the addition of regularization terms, resulting in optimization problems of the form 
\begin{subequations}\label{eq:fairContinuousRegularization}
\begin{align}
     \min_{\coef \in \R^n} & \sum_{i=1}^m\mathcal{L}\left(\coef^\top\bm{x}_i, y_i\right)+\mu\sum_{\ell=1}^n r(w_\ell)\\
\text{s.t.}\;&\max_{j\in [\ell]}\left| \frac{1}{m_1}\sum_{i=1: a_i = 1}^m \mathbbm{1}(\coef^\top\bm{x}_i > b_j) -\frac{1}{m} \sum_{i=1}^m \mathbbm{1}(\coef^\top\bm{x}_i > b_j)\right|\leq \epsilon,
\end{align}
\end{subequations}
where $\mu\geq 0$ and $r:\R\to\R_+$ is a suitable regularization function. Common choices of $r$ are the ridge penalty \citep{hoerl1970ridge} with $r(w_\ell)=w_\ell^2$, which induces robustness, the $\ell_1$ penalty \citep{tibshirani1996regression} with $r(w_\ell)=|w_\ell|$, which promotes sparsity, and the MCP/reverse Huber penalty/perspective regularization \citep{xie2020scalable,dong2018regularization,zhang2010nearly} with $r(w_\ell;d)=\min_{0\leq\zeta\leq 1}w_\ell^2/\zeta+d\zeta$ for some parameter $d>0$, which induces both robustness and sparsity. When the regularization function $r$ is convex, as is the case for the ridge, $\ell_1$ bad MCP/reverse Huber examples, it can be directly incorporated into the objective of \eqref{mod.strong.formulation} to yield strong convex relaxations.

Although we focus on demographic parity, the proposed methods are flexible and can accommodate notions of fairness based on conditional distributions over predictions \citep{rychener2022metrizing}, such as equal opportunity and equalized odds. Another flexibility is the norm used to measure deviations from perfect fairness. Notice that \eqref{def.dp} corresponds to the $L_\infty$ norm of the difference between the conditional and marginal cumulative distribution, also called the Kolmogorov-Smirnov statistic \citep{agarwal2019fair}. Other common norms that could be considered are the $L_1$ norm (corresponding to Wasserstein distance \citep{ye2024distributionally}) and the squared $L_2$ norm (corresponding to Cramér distance \citep{rychener2022metrizing}).

Finally, the problem formulation can easily be adapted to generalized linear models \citep{nelder1972generalized}, other measures of fairness, as well as accommodate multiple protected classes with minimal changes. Generalized linear models make predictions by transforming the linear output $\coef^\top\bm{x}$ by means of a function $g$: $\hat{y} = g^{-1}\left(\coef^\top\bm{x}\right)$. Generalized linear models go beyond labels that are real-valued, and can handle binary, categorical, and count outcomes, covering ubiquitous models such as least squares, logistic, multiclass, and Poisson regression. The only change required to accommodate a generalized linear model is to use 
indicators over non-linear inequalities  to compute empirical fairness
: $\mathbbm{1}(g^{-1}(\coef^\top\bm{x}_i) > b))$, which can be conveniently linearized by converting the condition to an equivalent linear one, $\mathbbm{1}(\coef^\top\bm{x}_i > g(b))$.

% \paragraph{Other methods}
% \citet{agarwal2019fair} propose to heuristically solve the fair regression problem by constructing a sequence of non-convex cost-sensitive classification problems. Each problem is then heuristically optimized using convex approximations. Similar to this work, the methods developed by \citet{ye2024distributionally} are based on MIO techniques. They propose an MIO formulation for the related problem of Wasserstein-fair regression, where they provide exact formulations and heuristic algorithms based on convex relaxations. A key difference in their approach is that the convex relaxations they consider are based on convexifying the Wasserstein-fairness constraint without considering the objective function, making it complementary to our approach. 
\section{Computational experiments} \label{sec:experiments}

In this section, we present computational experiments to test the optimization and empirical performance of the fair regression models trained using the three proposed approaches described in \S\ref{sec:comparisons}: the mixed integer formulation, the strong convex relaxation, and coordinate descent. 
First, we describe the synthetic and real datasets used for experimental results in \S\ref{sec:numerical.datasets}. Then, in \S \ref{sec:optimization.performance}, we present numerical results on synthetic data comparing the optimization performance of the proposed methodologies for least squares regression with each other and the \basicmodel\ formulation. Our findings indicate that although the branch-and-bound produces high-quality results, it does not currently scale beyond small instances. In contrast, the relaxation and coordinate descent methods promote fairness in a fraction of the time. In \S \ref{sec:computationsReal}, we present results on real datasets for both least squares and logistic regression, comparing the relaxation and coordinate descent with the state-of-the-art approaches from the literature.
All experiments were run in Python 3.11 using Gurobi 10.0.0 solver for least squares regression problems and MOSEK 10.2.3 for logistic regression problems with a one-hour time limit.

\subsection{Datasets} \label{sec:numerical.datasets}
We describe the synthetic and real datasets used to perform experiments.
\paragraph{Synthetic data.}
We generate a synthetic dataset following the method outlined in \citet{ye2024distributionally}. For a given number of features $n$, the data is generated from $y = \coef^\top \bm{x} + \varepsilon$ with the first $\left\lfloor n/2 \right\rfloor$ elements of $\coef$ sampled i.i.d from $\mathrm{Unif}(-1,0)$ and the next $\left\lfloor n/2\right\rfloor$ elements sampled i.i.d. from $\mathrm{Unif}(0,10)$. The last element is set to zero, $w_n = 0$. Given a population size $m$, the first $\left\lceil 3m/4\right\rceil$ observations are assigned sensitive attribute $a=0$, and the remainder minority population $a=1$. Labels are scaled to a normalized range $[0,1]$. We fix the number of features to be $n=10$ and for each $m\in\{15, 30, 50, 100\}$ observations, we generate five random instances.
\paragraph{Real data for least squares models.}
\begin{itemize}
    \item \emph{Communities \& Crime} \citep{communities}: This dataset contains 127 features measuring socio-economic, law enforcement, and crime statistics for 1994 communities in the United States. The objective is to build a model that predicts the rate of violent crimes per 100,000 population, with race as the sensitive attribute.
    \item \emph{Law School} \citep{lawschool}: A dataset of 20,649 records of law school students, tracking their GPA (normalized to be in $[0,1]$) and 12 features related to students' profile. The goal is to build a model that accurately predicts an applicant's GPA after enrollment while maintaining fairness with respect to race. Following the experimental set-up of \citet{agarwal2019fair}, we perform experiments on the entire dataset and a subsampled version, which is generated by randomly sampling 2,000 observations.
    \end{itemize}
\paragraph{Real data for logistic regression models.}
    \begin{itemize}
    \item \emph{Adult} \citep{adult}: A dataset comprised of 103 socioeconomic features per individual. The objective is to accurately predict the probability of an individual having an annual salary that exceeds \$50k while maintaining fairness with respect to sex. We randomly sample 2,000 observations.
    \item \emph{CelebA} \citep{celebA}: This dataset consists of feature vectors of $n=38$ binary features associated with celebrity images to predict whether the image is of someone smiling or not. Sex is used as the protected feature. We randomly sample $1,000$ observations.
\end{itemize}
Each dataset is randomly split evenly into training and testing sets.

\subsection{Computational performance of proposed formulations and algorithms}\label{sec:optimization.performance}
In this section, our goal is to evaluate how effectively the proposed methods can solve or approximate the fair regression problem \eqref{eq:fairContinuousRegAbs}. Here we focus on solution time and objective value, and defer the study of the statistical out-of-sample performance to \S\ref{sec:computationsReal}. As the exact mixed-integer optimization formulations do not scale to $m\approx 1,000$, the computations in this section are performed on synthetic datasets with fewer observations.
%We begin by evaluating the performance of the relaxations, algorithms, and exact methods proposed in this work. 
To this end, we conduct a series of experiments focused on solving a least squares regression problem with fairness regularization using the proposed methods.
The thresholds used to approximate exact fairness are $b_{i+1} = i\times (1/40), i = 0, \ldots, 40$, creating 40 equispaced intervals in $[0,1]$ used to measure fairness by $\widehat{\mathrm{DP}}_{41}$.
\paragraph{Methods}
In this section, we test the fair regression estimators $\bm{w}$ obtained by the following methods: 

\begin{itemize}
    \item \texttt{Big-M}: the natural mixed-integer \basicmodel~formulation \eqref{mod.bigM}.
    \item \texttt{MICQO}: the mixed integer conic quadratic strong formulation described in \S\ref{sec:algMIP} for least squares regression.
    \item \texttt{Relax}: the proposed strong convex relaxation \eqref{mod.strong.formulation} described in \S\ref{sec:algConvexRelax}.
    \item \texttt{CD-relax}: coordinate descent described in \S\ref{sec:algCoordinateDescent} initialized with the solution obtained by \texttt{Relax}.
    \item \texttt{CD-acc}: coordinate descent initialized with the solution to the vanilla least squares regression problem (most accurate but not fair).
    \item \texttt{CD-fair}: the coordinate descent algorithm initialized with the trivially fair solution, $\bm{w}=\bm{0}$ (most fair but not accurate).
\end{itemize}

Since there is randomness in the coordinate-wise locally optimal solutions found by coordinate descent, we solve coordinate descent for each instance with the same initial solution five times with different random seeds and record the best solution obtained. 
% \paragraph{Experiment setup} We generate synthetic datasets as described  in \S\ref{sec:numerical.datasets} with $n=10$ features, $m\in\{15, 30, 50, 100\}$ observations, and five instances per choice of parameters. A fair regression model is trained for each instance, with increasing regularization strength $\lambda$.\todo{Why not send this to \S\ref{sec:numerical.datasets}?}
\subsubsection{MIO formulations}\label{sec:computationsMIO}
The first set of experiments compares the optimization performance of \texttt{Big-M} and \texttt{MICQO}. In Table~\ref{tab:computational.results.synthetic}, we report the averages of the following measurements for 20 instances per row (five per value of $m$): \texttt{root gap}, calculated as $(\mathrm{obj}_{\mathrm{UB}}-\mathrm{obj}_{\mathrm{relax}})/\mathrm{obj}_{\mathrm{UB}}$, which is the relative gap between the continuous relaxation of the formulation and the best-known objective, the \texttt{end gap} reported by the solver after the time limit has been reached, the \texttt{time} the solver took to perform branch-and-bound in seconds, and the number of branch and bound \texttt{nodes} explored within the time limit.

The experiments show that the problem becomes harder to solve as greater fairness is demanded from the model. Unsurprisingly, the \texttt{Big-M} root gap is large, over 50\% for almost all instances, which negatively affects the performance of the branch-and-bound algorithm. The evidence of poor branch-and-bound performance is seen in the excessive run times and node count. \texttt{MICQO} shows a substantial improvement over the \texttt{Big-M} formulation. It has a significantly stronger relaxation at the root node, with an average root gap of $16.9\%$, compared to $76.9\%$ for \texttt{Big-M}, resulting in smaller branch-and-bound trees and faster run times to prove optimality. The average end gap of 4.9\% reveals that the best solution is of high quality even for instances where optimality is not verified within the time limit. Note that the root gap for the strong convex relaxation of \texttt{MICQO} is obtained in merely 0.25 seconds on average and is already substantially smaller that the end gaps of the \texttt{big-M} method. This suggests that the strong convex relaxation model may have statistical merit on its own, which we demonstrate in subsequent experiments presented in this section.

\begin{table}[h]
\small
\centering
\caption{Performance of \texttt{Big-M} and \texttt{MICQO} for fair linear regression: synthetic data.}
\begin{tabular}{c rrrr rrrr}
\toprule
\multirow{2}{*}{$\lambda$} & \multicolumn{4}{c}{\texttt{Big-M}} & \multicolumn{4}{c}{\texttt{MICQO}} \\ 
\cmidrule(lr){2-5} \cmidrule(lr){6-9}
 & \shortstack{Root \\Gap} & \shortstack{End\\Gap} & Time & Nodes & \shortstack{Root\\Gap} & \shortstack{End\\Gap} & Time & Nodes \\
\midrule
0.01 & 45.3\% & 2.1\%  & 487 & 4,010,293 & 1.9\% & 0.0\%  & 252 & 223,581 \\
0.02 & 60.8\% & 9.0\%  & 1,654 & 19,895,563 & 5.4\% & 0.7\%  & 908 & 1,079,684 \\
0.04 & 73.2\% & 19.1\% & 2,040 & 20,405,258 & 8.3\% & 1.9\%  & 950 & 902,741 \\
0.05 & 76.6\% & 28.4\% & 2,681 & 29,063,625 & 11.3\% & 2.6\%  & 1,212 & 1,058,541 \\
0.06 & 78.8\% & 30.7\% & 2,759 & 28,508,133 & 13.1\% & 2.9\%  & 1,507 & 1,447,340 \\
0.08 & 81.5\% & 38.1\% & 3,023 & 32,654,651 & 14.0\% & 4.0\%  & 1,673 & 1,284,442 \\
0.10 & 83.4\% & 47.4\% & 3,196 & 35,020,673 & 16.0\% & 4.6\%  & 1,803 & 1,275,695 \\
0.20 & 87.7\% & 55.3\% & 3,147 & 28,630,313 & 22.9\% & 8.6\%  & 1,625 & 683,419 \\
0.30 & 89.5\% & 62.8\% & 3,295 & 29,468,759 & 31.2\% & 10.8\% & 1,555 & 507,315 \\
0.50 & 91.7\% & 73.6\% & 3,600 & 31,892,472 & 45.5\% & 12.7\% & 1,611 & 437,835 \\
\midrule
\textbf{Avg} & \textbf{76.9\%} & \textbf{36.6\%} & \textbf{2,588} & \textbf{25,954,974}& \textbf{16.9\%} & \textbf{4.9\%} & \textbf{1,310} & \textbf{890,059} \\
\bottomrule
\end{tabular}
\label{tab:computational.results.synthetic}
\end{table}

\subsubsection{Coordinate descent}\label{sec:computationsCoordinate}

Despite the improvement over the basic big-$M$ formulation, \texttt{MICQO} is still computationally expensive for small instances with $n=10, ~m=100$. Motivated by the small relaxation gap for the strong formulation, we investigate the ability of the coordinate descent method to improve heuristically obtained solutions as an alternative to branch-and-bound. Table~\ref{tab:cd.performance} reports the \textbf{optimality gaps} for fair regression by the convex relaxation, coordinate descent algorithm, and the strong MIO formulation. The optimality gap is the relative gap between the best-known lower bound for the problem and the upper bound found by some model $\bm{w}$: $(\mathrm{obj}_{\bm{w}} - \mathrm{obj}_{\mathrm{LB}})/\mathrm{obj}_{\bm{w}}$. On average, coordinate descent terminates within two seconds.

Table~\ref{tab:cd.performance} confirms increasing difficulty in providing provably high-quality solutions as the fairness requirement becomes more stringent. We observe that coordinate descent is able to substantially improve upon an initial solution: \texttt{CD-relax} halves the optimality gap achieved by \texttt{Relax}. The initial solution used as input to the coordinate descent algorithm affects the quality of the outputted model. A consistent pattern is observed: \texttt{CD-relax} dominates \texttt{CD-acc}, which itself dominates \texttt{CD-fair}. \texttt{MICQO} results in higher-quality solutions, achieving an average optimality gap $3\times$ less than the second best method \texttt{CD-relax}, but is substantially more expensive even in datasets with few datapoints.
%\todo{How's the upper bound for Relax obtained?}

\begin{table}[H]
    \centering
    \caption{Optimality gaps with convex relaxation, coordinate descent starts, and branch-and-bound. \texttt{Relax} and  \texttt{CD-relax} terminate within 0.25 seconds and two seconds respectively, while \texttt{MICQO} may hit the one-hour time limit. \texttt{CD-relax} achieves an attractive trade-off between quality and computational time.}
\begin{tabular}{c|rrrrr}
\toprule 
 $\lambda$ & \texttt{Relax}\% & \texttt{CD-fair}\% & \texttt{CD-acc}\% & \texttt{CD-relax}\% & \texttt{MICQO}\%\\
 \hline
0.01 & 7.9 & 6.3 & 3.5 & 2.1 & 0.0 \\
0.02 & 13.4 & 9.0 & 5.9 & 4.1 & 0.7 \\
0.04 & 19.5 & 13.6 & 10.0 & 7.8 & 1.9 \\
0.05 & 24.0 & 16.1 & 10.9 & 8.6 & 2.6 \\
0.06 & 24.6 & 16.3 & 11.9 & 9.3 & 2.9 \\
0.08 & 26.3 & 22.2 & 17.0 & 13.3 & 4.0 \\
0.10 & 32.1 & 23.9 & 21.7 & 15.8 & 4.6 \\
0.20 & 44.1 & 37.1 & 31.0 & 24.6 & 8.6 \\
0.30 & 50.0 & 57.5 & 38.3 & 28.9 & 10.8 \\
0.50 & 56.2 & 65.4 & 44.2 & 36.5 & 12.7 \\
\hline
\textbf{Avg} & \textbf{29.8} & \textbf{26.7} & \textbf{19.4} &\textbf{15.1} & \textbf{4.9} \\
\bottomrule 
\end{tabular}
    \label{tab:cd.performance}
\end{table}

% We now compare the statistical performance of various models on the synthetic instances with $m=100$, as they are the most challenging\todo{among synthetic instances, or including real data as well?}. Figure~\ref{fig:synth100} plots the \textbf{fairness}, measured as $\widehat{\mathrm{DP}}$, and the \textbf{accuracy}, measured as the mean squared error, obtained by each model.\todo{It is strange to call this ``statistical performance", since you are comparing in-sample. How is this different from the previous section?} 
In Figure~\ref{fig:synth100}, we show the disaggregated objective terms measuring accuracy (mean squared error) and fairness ($\widehat{\mathrm{DP}}_{41}$) of various models on instances with $m=100$, as they are the most challenging.
Methods that outperform others in terms of optimality gap in Table~\ref{tab:cd.performance}, also excel in both accuracy and fairness. Although \texttt{MICQO} produces the best models in terms of solution quality, it requires substaintial computation time, often reaching the one-hour time limit. In contrast, \texttt{CD-relax} achieves high-quality solutions in a few seconds, making it a more practical choice for larger problems. Similarly, \texttt{Relax} obtains solutions that perform well in regimes with milder fairness requirements, an order of magnitude faster than \texttt{CD-relax}; nonetheless,
\texttt{Relax} struggles to produce estimators with $\widehat{\mathrm{DP}}_{41} < 0.07$. The limitation of convex approximations, which are often too relaxed to enforce stringent fairness levels, has been highlighted in \citet{lohaus2020too}. %Figure~\ref{fig:synth100} shows that \texttt{CD-relax} remedies this drawback by improving the fairness of a model trained via the convex relaxation. However, this improvement comes at the cost of reduced accuracy, which is expected since \texttt{CD} optimizes for exact fairness instead of underestimating it. 

% Finally, we observe that the relative improvement of the statistical performance of \texttt{MICQO} compared to other methods is most pronounced at the high-fairness regime, showing substantial improvement in accuracy. This indicates that \texttt{MICQO} is particularly effective at balancing fairness and accuracy in this challenging setting. We conclude that \texttt{MICQO} is able to produce fair and accurate solutions, albeit at the cost of increased computational effort.\todo{Reading your conclusions, it reads like every method is bad except MICQO.}

\begin{figure}[H]
    \centering
    \includegraphics[width=0.75\linewidth]{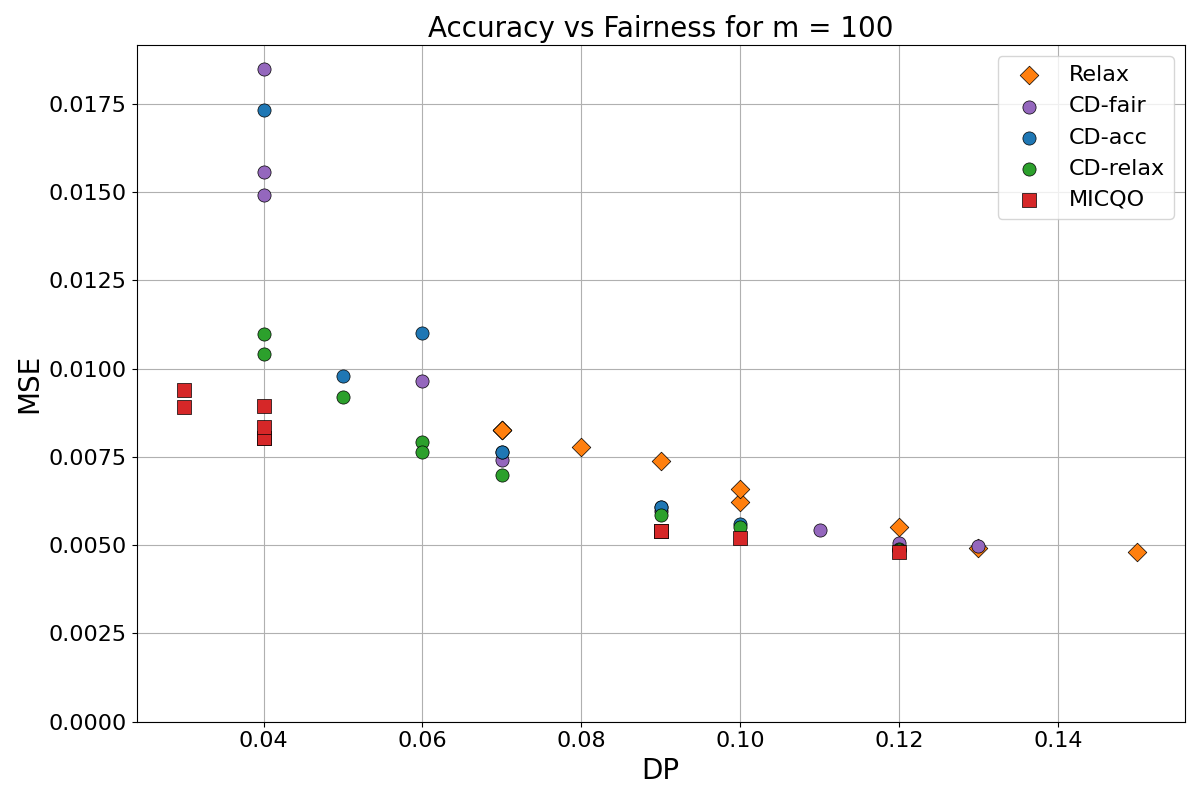}
    \caption{Accuracy (MSE) vs. fairness $\left(\widehat{\mathrm{DP}}_{41}\right)$ on synthetic data. \texttt{MICQO} achieves the best trade-off between accuracy and fairness but requires considerable computation time, often hitting the imposed one-hour time limit. On the other hand, \texttt{CD-relax} produces solutions that approach the quality of those attained by \texttt{MICQO} within seconds.}
    \label{fig:synth100}
\end{figure}

\subsection{Experiments with real benchmark data}\label{sec:computationsReal}
In this section, we evaluate the proposed methodologies for training fair regression models on real data and compare their statistical performance with the state-of-the-art methods from the literature. We present results for least squares regression in \S\ref{sec:realLeastSquares} and for logistic regression in \S\ref{sec:compFairLogistic}. Additionally, in \S\ref{sec:realClass}, we compare our methods discussed in \S\ref{sec:fairClassificationApproximation} with approximations from the literature for the special case of $\ell = 1$.
\paragraph{Methods} We consider the following methods for real data experiments:
\begin{itemize}
    \item \texttt{Relax}: the strong convex relaxation.
    \item \texttt{Relax + $L_2$} and \texttt{Relax +  Sparse}: the strong convex relaxation with $L_2$ regularization and the strong convex relaxation with perspective regularization, as discussed in \S\ref{sec:extensions}.
    \item \texttt{CD-relax} and \texttt{CD-relax + $L_2$}: coordinate descent with and without $L_2$ regularization.
    \item \texttt{FR-Reduction}: the reduction-based algorithm for fair regression introduced by \citet{agarwal2019fair}.
    \item \texttt{Linear-Approx} and \texttt{Convex-Approx}: the approximations for $\widehat{\mathrm{DP}}_1$ proposed by \citet{zafar2017fairnessB} and \citet{wu2019convexity} respectively discussed in \S\ref{sec:fairClassificationLiterature}.
\end{itemize}
We evaluate the performance of the different methodologies by comparing the accuracy and fairness scores achieved by each model. To measure accuracy, we evaluate the \textbf{relative loss increase}, which represents the increase in loss relative to the vanilla regression model (the most accurate but least fair model), computed as $100\times \frac{\mathrm{loss} - \mathrm{loss}_\mathrm{unfair}}{\mathrm{loss}_\mathrm{unfair}}$. This metric quantifies the additional error incurred by imposing fairness over the baseline. We also present both $\widehat{\mathrm{DP}}_{\ell}$ and $\widehat{\mathrm{DP}}$ as defined in \eqref{def.dp.hat} to evaluate the model fairness scores.

\subsubsection{Fair least squares regression}\label{sec:realLeastSquares}
We evaluate the performance of the proposed heuristic \texttt{Relax} on two real datasets commonly used to test fair least squares regression and compare with \texttt{FR-Reduction}. Specifically, we consider constrained fairness problems of the type \eqref{fair.training.constr}, and compare models solved by $\texttt{Relax}$ and \texttt{FR-Reduction} for the same $\epsilon$ values.
We solve two additional convex relaxations, \texttt{Relax} + $L_2$ and \texttt{Relax+Sparse}, to explore the impacts of regularization for model generalization on top of fairness constraints. 
%The first one solves the fairness-constrained least squares regression problem, with an added $L_2$ squared penalty, referred to as \texttt{Relax} + $L_2$. The second one incorporates an $L_2$ squares penalty norm, as well as a convex sparsifying constraint \citep{xie2020scalable}, and is denoted as \texttt{Relax+Sparse}. 
The sparsifying perspective regularizer is tested only on \emph{Communities \& Crime}, as the \emph{Law School} dataset has only twelve features.
%The performance of each model is evaluated using the following metrics:\todo{Looks weird that the format is different from before.}

Figures~\ref{fig:communities} and \ref{fig:law} show the accuracy-fairness curves achieved for on the \emph{Communities \& Crime}, sub-sampled \emph{Law School} and the entire \emph{Law School} datasets, respectively. Each figure consists of two rows: the first row shows the Pareto-frontier obtained on the training, while the second row shows results on the testing dataset. The first column of each figure shows the results as measured by the discretized fairness metric $\widehat{\mathrm{DP}}_{41}$, while the second column shows the exact fairness $\widehat{\mathrm{DP}}$.

Before discussing the performance of these models, it is worth noting that the discretized fairness metric approximates exact fairness quite effectively. Regarding the statistical performance, we observe that \texttt{Relax} is competitive with \texttt{FR-Reduction} on average across all datasets. The extent of the improvement of \texttt{Relax} varies by dataset, with the most significant gains observed in the \emph{Communities \& Crime} experiment. Moreover, regularization seems to be particularly beneficial for this dataset. 

\begin{figure}[t]
    \centering
    \includegraphics[width=\textwidth]{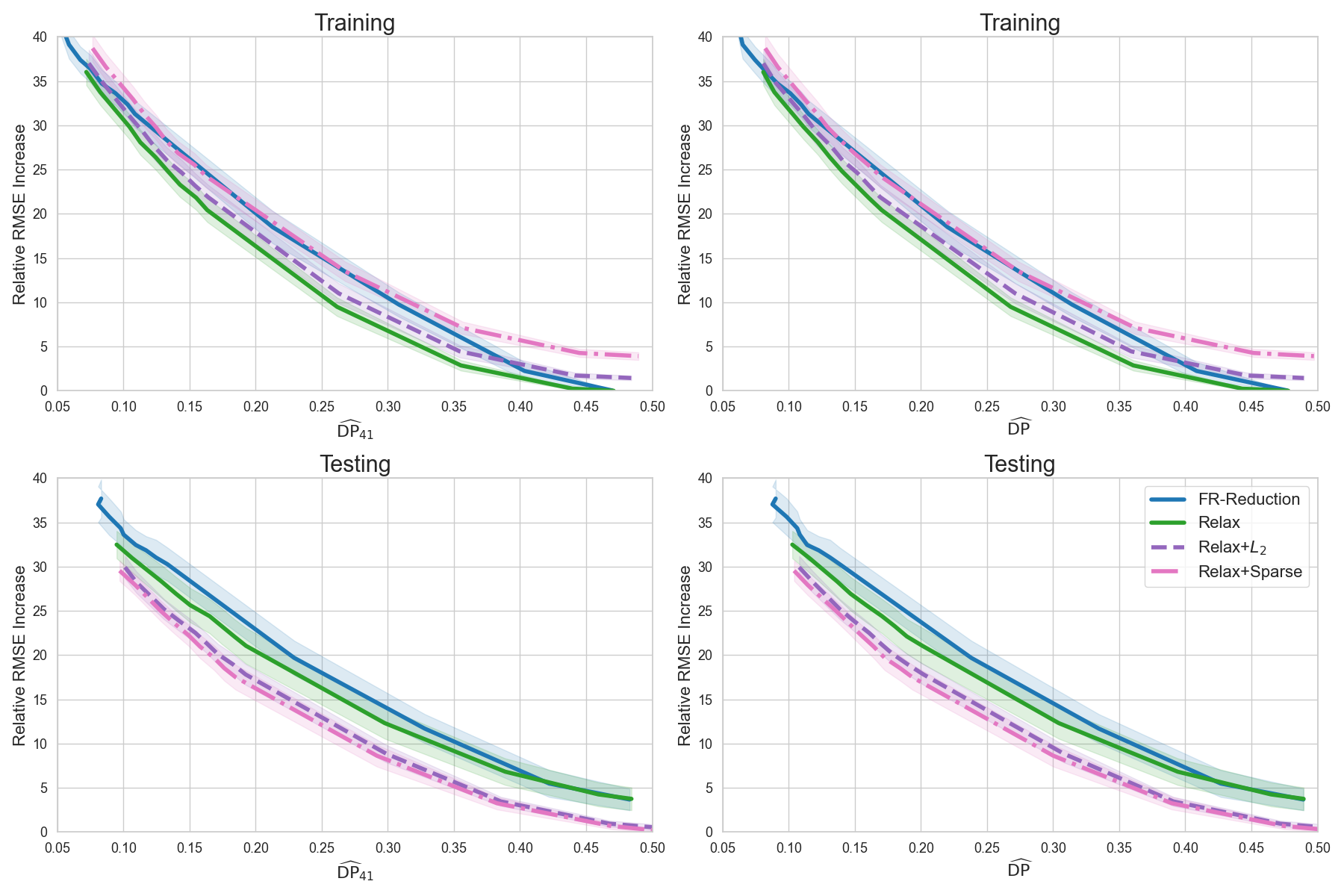}
    \caption{Accuracy-fairness trade-off curves obtained by models trained using \texttt{FR-Reduction} and \texttt{Relax} on the \emph{Communities \& Crime} dataset. Each curve represents the mean over 10 trials, with a 95\% confidence interval band on the relative RMSE. \texttt{Relax} and regularized variants produce least squares regression models that are competitive with or outperform the state-of-the-art in terms of out-of-sample performance with over $30\times$ improvement in runtimes.}
    \label{fig:communities}
\end{figure}

The average improvement of \texttt{Relax} over \texttt{FR-Reduction} is less pronounced in the experiments on the \emph{Law  School} datasets. An observation that aligns with previous results can be made, where \texttt{Relax} effectively improves fairness up to a certain threshold. Additional tests conducted with progressively smaller $\epsilon$ values show that the convex relaxation eventually converges to a model, and no further fairness improvement can be achieved using \texttt{Relax} alone. To overcome this limitation, we utilize \texttt{CD-relax} and present the results. The coordinate descent method produces models with lower $\widehat{\mathrm{DP}}$ than the fairest model produced by the convex relaxation. Moreover, it produces models with better accuracy for the same levels of $\widehat{\mathrm{DP}}$, improving the trade-off curve between fairness and accuracy. 

Table~\ref{tab:runtime_comparison} shows the computational time (in seconds) required by each method to train a model. Notably, we observe an average speed-up of $33\times$ over \texttt{FR-Reduction} when using \texttt{Relax} to train linear regression models. This substantial improvement is expected as \texttt{Relax} solves a single convex problem, whereas \texttt{FR-Reduction} requires solving a sequence of problems. For both methods, the computational time increases as fairness constraints become more stringent. 

A subtle but critical difference between our approach and the benchmark is that the model produced by \texttt{FR-Reduction} is a randomized predictor made up of $N$ models, $\bm{w}^1, \bm{w}^2, \ldots, \bm{w}^{N}$, each one with an associated probability $p_1, p_2, \ldots, p_{N}$. The reported accuracy and fairness metrics for \texttt{FR-Reduction} are calculated by taking the predictor's expected accuracy and expected fairness. While randomized predictors can achieve improved Pareto-frontier curves, they lack transparency and interpretability. These limitations may be particularly concerning in sensitive or legal contexts with fairness considerations. Upon further investigation, we found that the randomized predictor produced by \texttt{FR-Reduction} combines low-probability models that are fair but inaccurate and high-probability models that are accurate but less fair. In contrast, the proposed methodologies in this work provide deterministic models, offering greater interpretability.

\begin{figure}[t]
    \centering
    \begin{tabular}{cc}
        \includegraphics[width=0.5\textwidth]{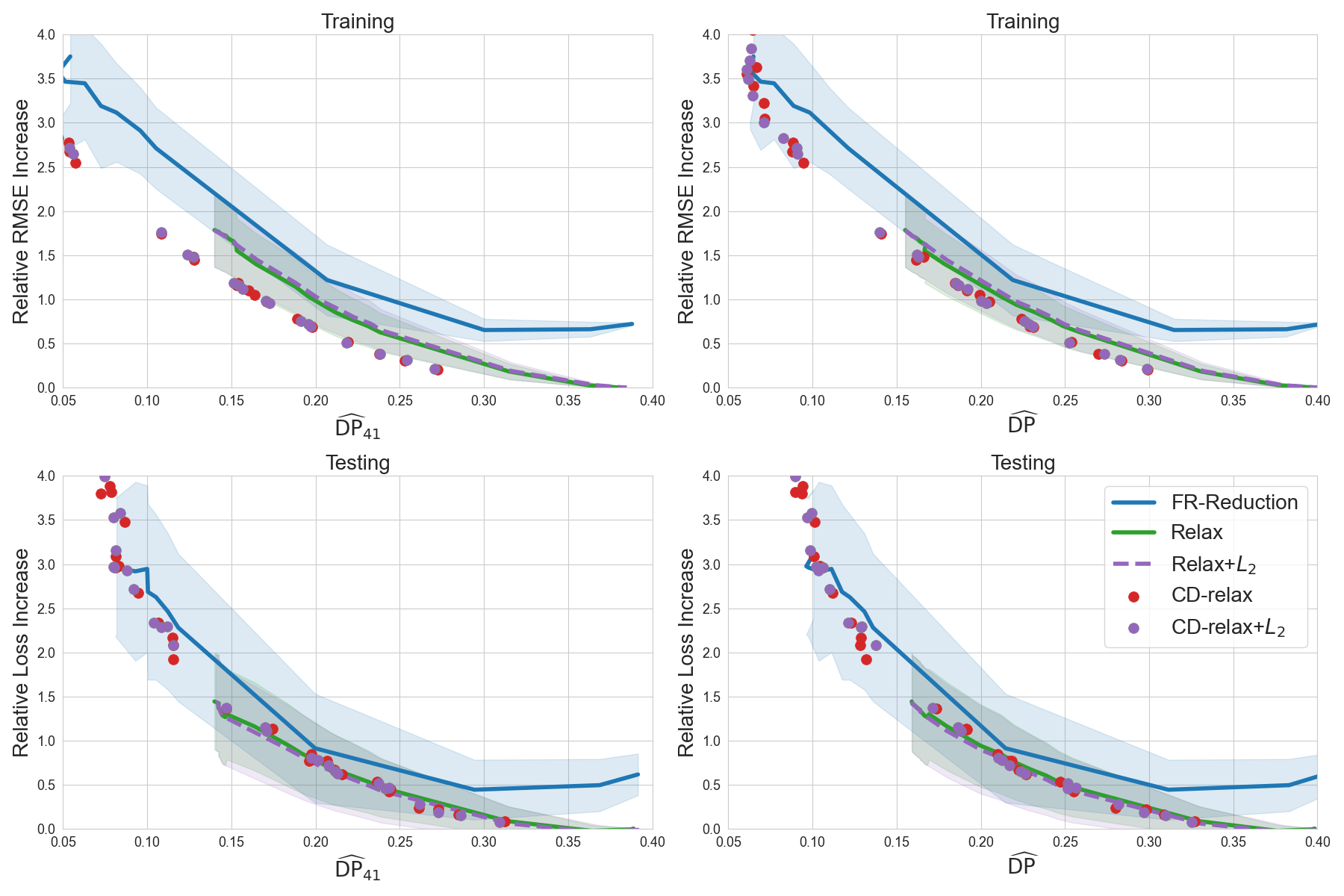} &  \includegraphics[width=0.5\textwidth]{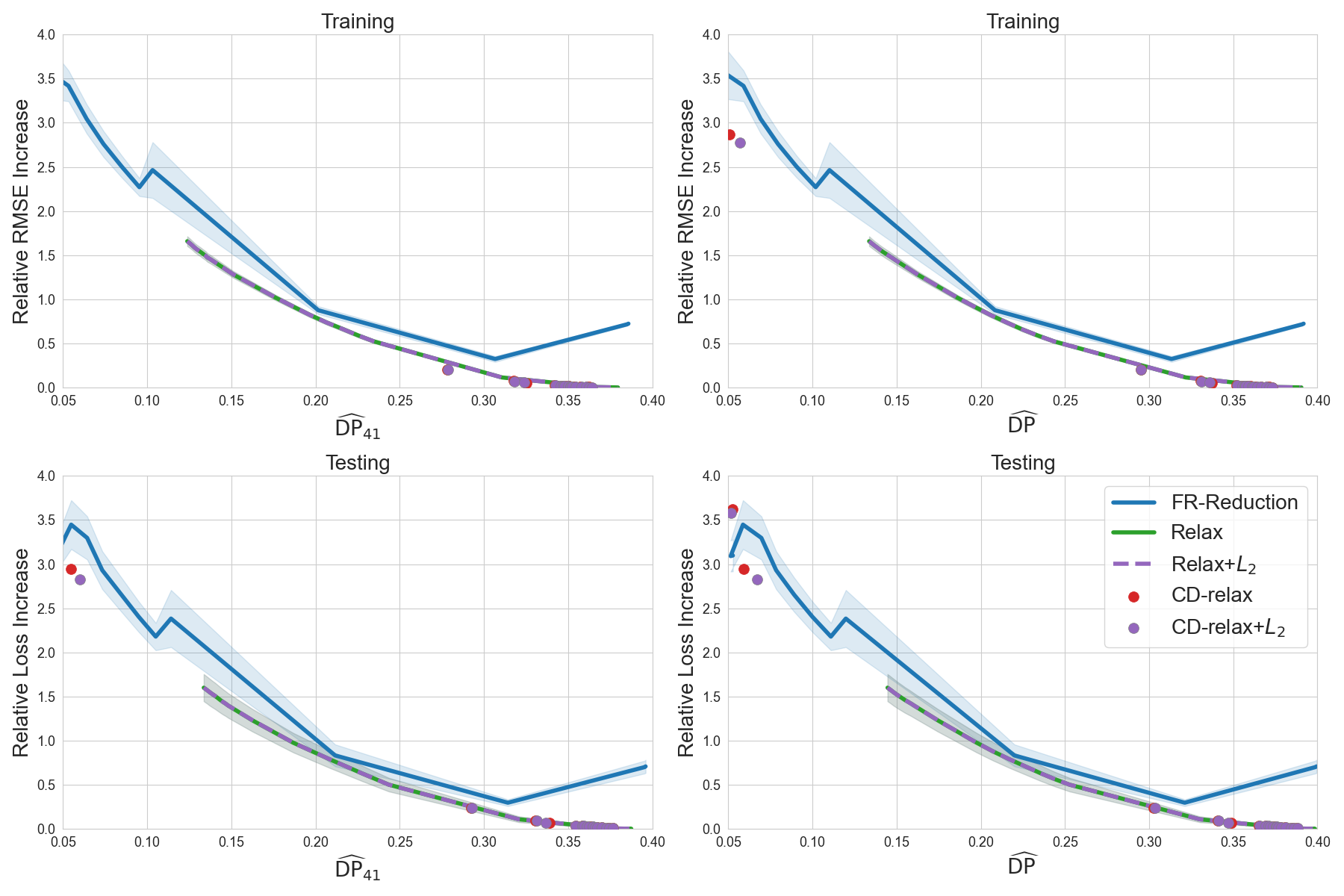} \\
        (a) Sub-sampled  & (b) Full dataset
    \end{tabular}
    \caption{Accuracy-fairness trade-off curves obtained by models trained using \texttt{FR-Reduction} and \texttt{Relax} and \texttt{CD-relax} on the sub-sampled and full \emph{Law School} dataset.}
    \label{fig:law}
\end{figure}

% \begin{figure}[h]
%     \centering
%     \begin{tabular}{cc}
%         \includegraphics[width=0.5\textwidth]{plots/law20649/law_DP_disc_train_20649.png} & \includegraphics[width=0.5\textwidth]{plots/law20649/law_DP_cont_train_20649.png} \\
%         (a) Training: Approx & (b) Training: Exact \\
%         \includegraphics[width=0.5\textwidth]{plots/law20649/law_DP_disc_test_20649.png} & \includegraphics[width=0.5\textwidth]{plots/law20649/law_DP_cont_test_20649.png} \\
%         (c) Testing: Approx & (d) Testing: Exact \\
%     \end{tabular}
%     \caption{Accuracy-fairness trade-off curves obtained by models trained using Agarwal and \texttt{Relax} and \texttt{CD-relax} on the full \emph{Law School} dataset.}
%     \label{fig:law.full}
% \end{figure}

\begin{table}[H]
\footnotesize
\centering
\caption{Computational time (s) for \texttt{FR-Reduction} and \texttt{Relax} to solve linear regression for real data. Both \texttt{Relax} and \texttt{CD-relax} are significantly faster than alternatives from the literature.}
\begin{tabular}{c rr rr rr rr rr rr}
\toprule
\multirow{2}{*}{$\epsilon$} & \multicolumn{2}{c}{Communities ($m=1,994$)} & \multicolumn{3}{c}{Law School ($m=2,000$)} & \multicolumn{2}{c}{Law School ($m=20,649$)} \\ 
\cmidrule(lr){2-3} \cmidrule(lr){4-6} \cmidrule(lr){7-9}
 & \ \ \texttt{FR-Reduction} & \texttt{Relax} & \ \ \texttt{FR-Reduction} & \texttt{Relax} & \texttt{CD-relax} & \ \ \texttt{FR-Reduction} & \texttt{Relax} & \texttt{CD-relax} \\
\midrule
0.01 & 406.76 & 5.25 & 197.90 & 4.68 & 25.78 & 2,246.24 & 62.33 & 664.47 \\
0.02 & 328.39 & 5.17 & 200.91 & 4.38 & 27.47 & 1,802.27 & 54.71 & 867.71 \\
0.04 & 226.57 & 5.05 & 164.97 & 4.04 & 29.22 & 1,564.63 & 61.21 & 556.14 \\
0.05 & 195.04 & 5.14 & 147.50 & 4.09 & 51.64 & 1,506.03 & 60.18 & 551.50 \\
0.06 & 175.84 & 5.28 & 144.14 & 4.10 & 25.88 & 1,506.71 & 62.67 & 443.99 \\
0.07 & 175.93 & 4.93 & 144.12 & 3.96 & 33.54 & 1,516.28 & 56.72 & 678.62 \\
0.08 & 160.05 & 4.98 & 138.49 & 3.94 & 48.30 & 1,516.17 & 57.92 & 863.78 \\
0.09 & 143.09 & 4.86 & 132.61 & 3.80 & 37.52 & 1,516.31 & 58.48 & 842.36 \\
0.10 & 151.83 & 4.89 & 129.63 & 3.77 & 41.39 & 1,419.45 & 63.81 & 629.22 \\
0.20 & 146.59 & 5.17 & 121.76 & 3.61 & 50.09 & 1,357.38 & 65.18 & 552.85 \\
0.30 & 141.50 & 4.83 & 106.62 & 3.40 & 43.21 & 1,237.72 & 47.46 & 689.60 \\
0.40 & 156.23 & 4.44 & 83.93 & 3.26 & 137.85 & 645.76 & 39.43 & 492.56 \\
0.50 & 76.47 & 3.91 & 56.85 & 3.11 & 34.83 & 590.58 & 39.40 & 335.15 \\
\midrule
\textbf{Avg} & \textbf{191} & \textbf{4.9} & \textbf{136} & \textbf{3.9} & \textbf{45} & \textbf{1,417} & \textbf{56} & \textbf{628} \\
\bottomrule
\end{tabular}
\label{tab:runtime_comparison}
\end{table}

\subsubsection{Fair logistic regression}\label{sec:compFairLogistic}

% In this section, we present two experiments, one considering logistic regression ($\widehat{\mathrm{DP}}$) with and the other considering fair classification metrics ($\widehat{\mathrm{DP}}_0$). %\todo{why does the section title does not refer to fair classification only then?}

This section compares the logistic regression models trained using \texttt{FR-Reduction} with those derived by solving \texttt{Relax}. Figure~\ref{fig:adult} presents the accuracy-fairness curves for the \emph{Adult} dataset, where accuracy is measured by the relative average training and testing logistic loss for the models. Following the experiment design in \citet{agarwal2019fair}, the discretization used to train the model is the grid from $-5$ to $5$, consisting of 40 equispaced intervals, $[-5, -4.75, -4.50, \ldots, 5]$. Consistent with the findings in the least squares setting, we see improved performance in both training and testing phases with \texttt{Relax}. Moreover, \texttt{Relax} builds logistic regression models faster than the benchmark: on average, the benchmark takes 103 secs. to build a fair logistic regression model, compared to 50 secs. for \texttt{Relax}. \texttt{CD-relax} is able to produce models that achieve better levels of fairness, although at the expense of runtimes of 188 seconds on average.

\begin{figure}
    \centering
    \includegraphics[width=0.75\linewidth]{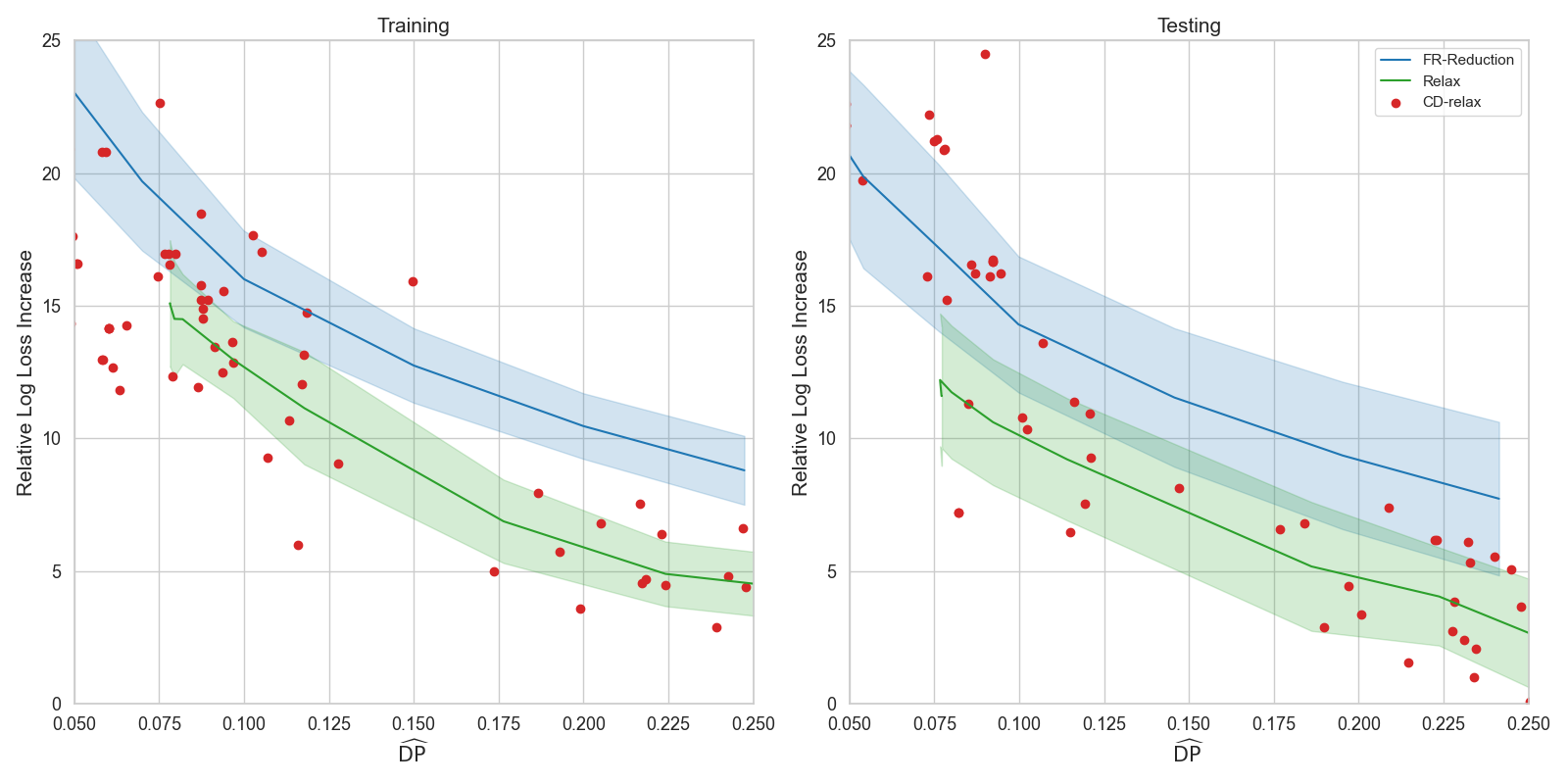}
    \caption{Accuracy-fairness trade-off curves obtained by models trained using \texttt{FR-Reduction}, \texttt{Relax} and \texttt{CD-Relax} on the \emph{Adult} dataset. \texttt{Relax} produces logistic regression models that are competitive with the state-of-the-art in terms of out-of-sample performance in half the computational runtime.}
    \label{fig:adult}

\end{figure}
\subsubsection{Computations with fair classification metrics}\label{sec:realClass}
This section presents results for solving logistic regression when fairness is considered at a single threshold, that is when the fairness metric can be represented by $\widehat{\mathrm{DP}}_1$. We compare the proposed methodologies with the approximations discussed in \S\ref{sec:fairClassificationLiterature}, \texttt{Linear-Approx} and \texttt{Convex-Approx}. Since $\widehat{\mathrm{DP}}_1$ only requires two binary variables per observation to formulate the mixed-integer optimization problem, we can solve the exact problem to optimality using the branch-and-bound method within a few minutes. Therefore, this exact approach is included as part of the numerical results. Experiments are conducted on logistic regression models for the \emph{CelebA} dataset.

We present an initial set of results comparing the model obtained by solving the convex relaxation with $\widehat{\mathrm{DP}}_1$ as a fairness-promoting regularizer to those trained using $\widehat{\mathrm{DP}}_\ell$, $\ell>1$. As explained in \S\ref{sec:fairClassificationApproximation}, training with $\widehat{\mathrm{DP}}_1$ can result in poor out-of-sample performance, which we attempt to mitigate by adding artificial thresholds to measure fairness in the training problem. We solve the convex relaxation of the logistic regression training problem with $\widehat{\mathrm{DP}}_\ell$ as a regularizer, dividing the range $[-3,3]$ into $\ell$ equispaced breakpoints, with $\ell \in \{3,5,9,15\}$ and present the results in Figure~\ref{fig:celebaA.stat.performance}. Observe that in both the training and the testing phases, the convex approximation with $\ell>1$ effectively mitigates the limitations noted for the convex models, achieving high levels of fairness. Interestingly, approximations with high granularity slightly outperform \texttt{MIO} in out-of-sample tests. 

\begin{figure}[t]
    \centering\includegraphics[width=\linewidth]{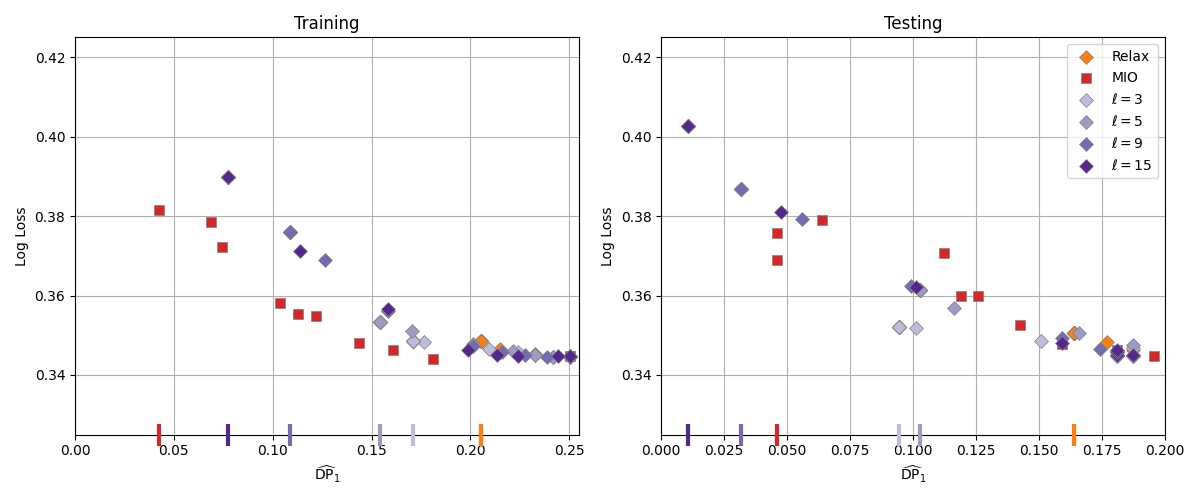}
        \caption{Accuracy vs fairness for logistic regression for of the \emph{CelebA} dataset trained using the proposed methodologies: mixed-integer optimization, the convex relaxation and the convex approximation using $\mathrm{\text{DP}}_\ell$ for $\ell>1$.}
    \label{fig:celebaA.stat.performance}
\end{figure}

\begin{figure}[h!]
    \centering
    \includegraphics[width=\linewidth]{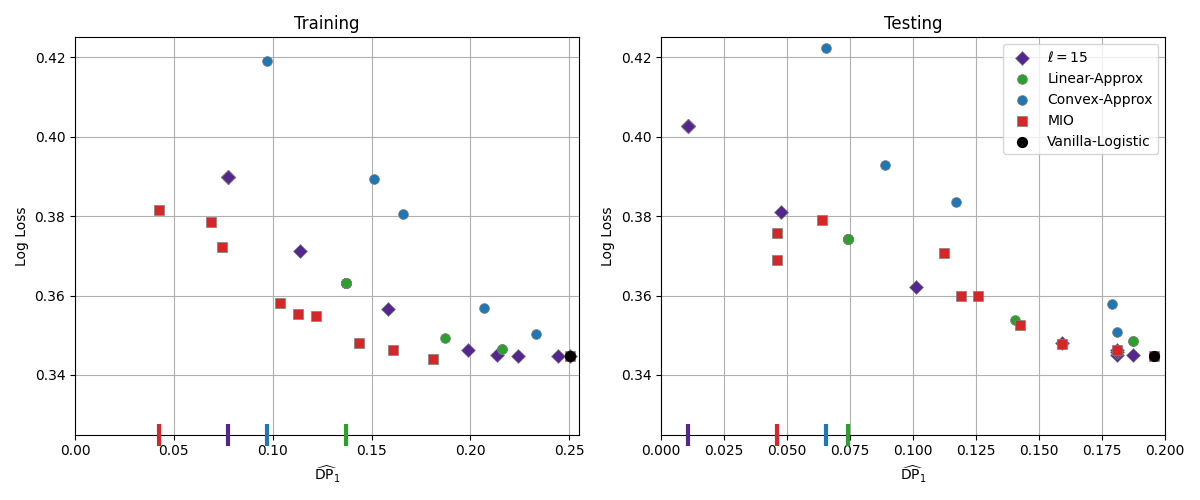} 
    \caption{Accuracy vs fairness for logistic regression for the \emph{CelebA} dataset. The black dot represents the vanilla logistic regression trained without any fairness considerations. Colored bars on the x-axis show the lowest $\widehat{\mathrm{DP}}_1$ attained by each method.}\label{fig:logistic.regression.stat.performance}
\end{figure}

Figure~\ref{fig:logistic.regression.stat.performance} shows the accuracy-fairness trade-off curves achieved over the training set for various methods. While \texttt{MIO} outperforms all convex methods in training, the convex approximation with $\widehat{\mathrm{DP}}_{15}$ performs similarly well out-of-sample, and is able to achieve high levels of fairness that other convex methods cannot. Notably, convex methods cannot achieve increasingly fair solutions beyond a certain threshold. We mark this threshold on the x-axis of Figure~\ref{fig:logistic.regression.stat.performance}.

\section{Conclusion}
We present a mixed-integer optimization framework for fair regression problems and propose a strong convex relaxation obtained by convexifying the key mixed-integer substructure used to define fairness. The strong relaxation is exact for single-observation and single-factor fair regression. We develop three new methods utilizing this convexification to train fair regression models, each balancing optimality and the computational effort. The first is a heuristic method that uses regression coefficients obtained by solving the strong convex relaxation as a standalone. The second is a coordinate descent method that improves upon a provided initial model. The third method employs the developed strong convex relaxation with branch-and-bound methods to solve the exact mixed-integer optimization problem. We demonstrate through numerical experiments on fair least squares and fair logistic regression problems that even the simplest method (using the strong convex relaxation as a standalone) builds regression models that outperform the state-of-the-art approaches in trading off accuracy and fairness in out-of-sample tests and does so at significantly lower computational effort.
%\newpage
\bibliography{bib}
\bibliographystyle{plainnat}
\appendix
\section{Appendix}
\subsection{Coordinate descent implementation details}

We describe an efficient implementation of a single iteration of the coordinate descent algorithm \ref{alg:cd} to compute $w_k^{t+1} = \arg\min_{w_k \in \tilde{B}}\mathcal{L}(w_k) + \lambda R(w_k)$, where $\mathcal{L}(w_k)$ is the loss function and $R(w_k)$ is the fairness regularize. The fairness regularizer is computed by solving the optimization problem \eqref{eq:R(w)}. At each iteration, we enumerate the objective of the $\mathcal{O}(m\times \ell)$ candidates in $\tilde{B}$, selecting the best solution. To reduce computation time, we exploit the structure of the fairness term, which can be dynamically updated if $\tilde{B}$ is sorted.

We now introduce some necessary notation. Let $\tilde{b}_{[s]}$ be the $s$-th smallest element of $\tilde{B}$, and $i(s),j(s)$ the corresponding indices, that is $\tilde{b}_{i(s)j(s)} = \tilde{b}_{[s]}$. Define $\bm{z}_s$ as the only feasible solution to \eqref{eq:R(w)}, and let $c_s = R(w_k)$ when $w_k \in \left(\tilde{b}_{[s-1]}, \tilde{b}_{[s]}\right)$. Introducing vector $\bm{d}(\bm{z}) \in \R^\ell$, where $d_j(\bm{z}) = \frac{1}{m_1} \sum_{i=1,a_i=1}^m z_{ij} - \frac{1}{m} \sum_{i=1}^mz_{ij}$, we can write $c_s = \|\bm{d}(\bm{z}_s)\|_\infty$.

Given $\bm{z}_{s-1}$, the solution $\bm{z}_{s}$ can be quickly obtained since only a single constraint in \eqref{eq:R(w)} changes signs when $w_k$ is restricted to be in $\left(\tilde{b}_{[s-1]}, \tilde{b}_{[s]}\right)$ instead of $\left(\tilde{b}_{[s-2]}, \tilde{b}_{[s-1]}\right)$: $$\bm{z}_s = \begin{cases} \bm{z}_{s-1} + \mathrm{sign}\left(x_{i(s)k}\right) \left(\frac{1}{m_1} - \frac{1}{m}\right)\bm{e}_{j(s)} & \text{if } a_{j(s)} = 1, \\
\bm{z}_{s-1} + \mathrm{sign}\left(x_{i(s)k}\right)\frac{1}{m}\bm{e}_{j(s)} & \text{else.}\end{cases}$$ 
This recursive relationship allows us to quickly compute fairness evaluated at every point in $\tilde{B}$ instead of solving \eqref{eq:R(w)} for every point. Assuming $x_{ik} \neq 0, i \in [m]$, we have $R\left(\tilde{b}_{[s]}\right) = \min\left(c_{s-1}, c_s\right) = \min\left(\|\bm{z}_{s-1}\|_\infty, \|\bm{z}_{s}\|_\infty\right)$, which can be rapidly computed if the elements of $\tilde{B}$ are sorted.

\subsection{Data Generation}\label{appendix.data}

Here, we describe the data generation process used for the comparisons presented in Figure~\ref{fig:approx} in \S \ref{sec:comparisons}. Following the procedure used in \citet{zafar2017fairnessB}, we generate $n=200$ synthetic observations. Each observation is assigned a binary label at random. For each observation, we generate a two-dimensional feature vector by sampling from Gaussian distributions conditioned on the assigned label: $p(\bm{x}|y=1) = N\left(\begin{bmatrix}
    2 \\ 2
\end{bmatrix}, \begin{bmatrix}
    5 & 1 \\
    1 & 5
\end{bmatrix}\right)$ and $p(\bm{x}|y=0) = N\left(\begin{bmatrix}
    -2 \\ -2
\end{bmatrix}, \begin{bmatrix}
    10 & 1 \\
    1 & 10
\end{bmatrix}\right)$. Finally, protected status is randomly assigned to each observation $i$ with probability $p(z=1) = p(\bm{A}\bm{x}_i|y=1)/(p(\bm{A}\bm{x}_i|y=1)+p(\bm{A}\bm{x}_i|y=0))$, where $\bm{A} = \begin{bmatrix}
    \cos(\pi/4) & -\sin(\pi/4)\\
    \sin(\pi/4) & \cos(\pi/4)
\end{bmatrix}$.
\end{document}